\PassOptionsToPackage{table}{xcolor}

\documentclass{article} 
\usepackage{iclr2027_conference,times}

\usepackage{amsmath,amsfonts,bm}

\def\eqref#1{equation~\ref{#1}}

\def\1{\bm{1}}

\DeclareMathAlphabet{\mathsfit}{\encodingdefault}{\sfdefault}{m}{sl}
\SetMathAlphabet{\mathsfit}{bold}{\encodingdefault}{\sfdefault}{bx}{n}

\usepackage{amsmath}
\usepackage{amssymb}

\usepackage{booktabs}
\usepackage{multirow}
\usepackage{makecell}
\usepackage{adjustbox}
\usepackage{minitoc}
\usepackage{titletoc}

\usepackage{graphicx}
\usepackage{wrapfig}
\usepackage{capt-of}

\usepackage{xcolor}

\definecolor{deltagray}{RGB}{247,247,247}

\definecolor{brickred}{HTML}{A94438}

\usepackage{listings}
\usepackage[most]{tcolorbox}

\usepackage{url}
\usepackage{hyperref}

\hypersetup{
    colorlinks=true,
    citecolor=brickred,
    linkcolor=brickred,
    urlcolor=brickred
}

\newcommand{\tablesize}{\fontsize{8.3}{8.1}\selectfont}

\newcommand{\dsize}{\fontsize{7.5}{6.2}\selectfont}

\definecolor{goodgreen}{RGB}{104,138,119}
\definecolor{badred}{RGB}{181,118,108}
\definecolor{neutralgray}{RGB}{135,135,135}

\newcommand{\dpos}[1]{%
    \textcolor{goodgreen}{\dsize +#1}%
}

\newcommand{\dneg}[1]{%
    \textcolor{badred}{\dsize #1}%
}

\newcommand{\dzero}{%
    \textcolor{neutralgray}{\dsize 0.0}%
}

\newcommand{\dposbf}[1]{%
    \textcolor{goodgreen}{\dsize\textbf{+#1}}%
}

\newcommand{\dnegbf}[1]{%
    \textcolor{badred}{\dsize\textbf{#1}}%
}

\newcommand{\dzerobf}{%
    \textcolor{neutralgray}{\dsize\textbf{0.0}}%
}

\definecolor{goodgreen}{RGB}{104, 138, 119}
\definecolor{badred}{RGB}{181, 118, 108}

\newcommand{\dgood}[1]{%
    \textcolor{goodgreen}{\dsize #1}%
}

\newcommand{\dbad}[1]{%
    \textcolor{badred}{\dsize #1}%
}

\newcommand{\dgoodbf}[1]{%
    \textcolor{goodgreen}{\dsize\textbf{#1}}%
}

\newcommand{\ddash}{%
    \textcolor{gray!60!black}{\dsize --}%
}

\newcommand{\dname}[1]{%
    \textcolor{gray!70!black}{%
        \dsize $\Delta_{\text{#1}}$%
    }%
}

\newcommand{\Ind}[1]{%
    \mathbf{1}\!\left\{#1\right\}%
}

\definecolor{systembg}{RGB}{247,248,250}
\definecolor{systemframe}{RGB}{185,190,200}

\definecolor{userbg}{RGB}{245,249,255}
\definecolor{userframe}{RGB}{150,180,220}

\definecolor{assistantbg}{RGB}{246,251,247}
\definecolor{assistantframe}{RGB}{145,190,155}

\definecolor{toolbg}{RGB}{255,249,241}
\definecolor{toolframe}{RGB}{220,175,115}

\tcbset{
    promptstyle/.style={
        enhanced,
        breakable,
        boxrule=0.6pt,
        arc=1.5mm,
        left=2.5mm,
        right=2.5mm,
        top=2mm,
        bottom=2mm,
        before skip=5pt,
        after skip=5pt,
        fontupper=\small,
        fonttitle=\small\bfseries,
        attach boxed title to top left={
            xshift=2mm,
            yshift=-1.5mm
        },
        boxed title style={
            boxrule=0pt,
            arc=1mm,
            left=1.5mm,
            right=1.5mm,
            top=0.5mm,
            bottom=0.5mm
        }
    }
}

\newtcolorbox{systembox}{
    promptstyle,
    colback=systembg,
    colframe=systemframe,
    colbacktitle=systemframe!25,
    coltitle=black,
    title=System
}

\newtcolorbox{userbox}{
    promptstyle,
    colback=userbg,
    colframe=userframe,
    colbacktitle=userframe!25,
    coltitle=black,
    title=User
}

\newtcolorbox{assistantbox}{
    promptstyle,
    colback=assistantbg,
    colframe=assistantframe,
    colbacktitle=assistantframe!25,
    coltitle=black,
    title=Assistant
}

\newtcolorbox{toolbox}{
    promptstyle,
    colback=toolbg,
    colframe=toolframe,
    colbacktitle=toolframe!25,
    coltitle=black,
    title=Tool
}

\lstdefinestyle{promptcode}{
    basicstyle=\ttfamily\scriptsize,
    breaklines=true,
    breakatwhitespace=false,
    columns=fullflexible,
    keepspaces=true,
    showstringspaces=false,
    frame=none,
    xleftmargin=0pt,
    xrightmargin=0pt,
    aboveskip=2pt,
    belowskip=2pt
}

\usepackage{tabularx}
\newcolumntype{Y}{>{\raggedright\arraybackslash}X}

\usepackage{amsthm}
\usepackage[most]{tcolorbox}

\definecolor{papoprop}{RGB}{153,62,55}

\newtheorem{proposition}{Proposition}

\newtcolorbox{propositionbox}{
    enhanced,
    breakable,
    colback=papoprop!4,
    colframe=papoprop!55,
    boxrule=0.7pt,
    arc=1.5mm,
    left=1.2mm,
    right=1.2mm,
    top=0.8mm,
    bottom=0.8mm,
    before skip=6pt,
    after skip=6pt
}

\title{Priority-Preserving Advantage Aggregation for Multi-Reward Policy Optimization}

\title{Priority-Preserving Advantage Aggregation for Multi-Reward Policy Optimization}

\author{
\begin{tabular}{@{}l@{}}
\textbf{Xiangchi Yuan}$^{1}$,
\textbf{Xingjian Diao}$^{2}$,
\textbf{Zheyuan Liu}$^{3}$,
\textbf{Chunhui Zhang}$^{2}$,
\textbf{Dachuan Shi}$^{1}$, \\[0.15em]
\textbf{Kejing Xia}$^{1}$,
\textbf{Zhenbang Du}$^{1}$,
\textbf{Shenglong Yao}$^{1}$,
\textbf{Soroush Vosoughi}$^{2}$,
\textbf{Wenke Lee}$^{1}$ \\[0.35em]
{\normalfont\small
$^{1}$Georgia Institute of Technology \quad
$^{2}$Dartmouth College \quad
$^{3}$University of Notre Dame}
\end{tabular}
}

\iclrfinalcopy

\iclrfinalcopy

\begin{document}

\maketitle

\begin{abstract}
LLM post-training often optimizes multiple objectives with unequal priorities, such as correctness and reasoning efficiency. However, the widely used additive reward aggregation in RL can violate these priorities: when primary and auxiliary advantages conflict, a sufficiently large auxiliary signal can overturn the preference induced by the primary objective. We term this problem \emph{advantage priority inversion}. Although reward-level gating can protect priority more directly, it still removes potentially useful auxiliary distinctions among rollout samples that fail the primary condition.
To address this trade-off, we propose Primary Advantage Policy
Optimization (\textbf{PAPO}), a priority-aware advantage aggregation
method that asymmetrically combines a primary and an auxiliary advantage.
PAPO retains auxiliary supervision when the objectives agree and smoothly
attenuates conflicting auxiliary signals, preserving the primary sign
during fusion for \emph{primary-consistent ordering}.
This rule applies to sequence-level advantages in GRPO
and token-level advantages in actor--critic PPO. Experiments on multiple domains and models show that PAPO consistently achieves the best primary objective
while maintaining best or competitive auxiliary-objective performance across both GRPO- and
PPO-based optimization. 
\href{https://xiangchi-yuan.github.io/papo-project/}{Code is available}.
\end{abstract}

\section{Introduction}

Reinforcement learning (RL) has become a central component of large language
model (LLM) post-training, substantially improving reasoning and other complex
capabilities~\citep{schulman2017proximal,guo2025deepseek,diao2026errors,yuan2025mitigating,yuan-etal-2026-behavior,shi2026swireasoning,shi2026copt}.
Many post-training tasks naturally involve more than one reward objective.
Existing methods commonly combine such signals through additive reward or
advantage aggregation, sometimes together with independent normalization or
adaptive weighting~\citep{ichihara2025awgrpo,diao2026addressing,liu2026gdpo,jiang2026dvao}.
However, these reward objectives are often not equally important.
In many settings, one primary objective is more important and should determine whether an update is favored,
while another auxiliary objective should only refine how that objective is achieved.
For example, correctness should take priority over reasoning length, and task
success should take priority over tool-use efficiency.

When such an
asymmetric objective relationship is present, existing multi-reward optimization faces three challenges 
\textit{(i) Additive aggregation can invert the primary preference.}
When the primary and auxiliary advantages conflict, a sufficiently large
auxiliary advantage can reverse the sign induced by the primary advantage,
causing the policy to reinforce a primary-disfavored response or suppress a
primary-favored one.
We refer to this problem caused by additive aggregation methods~\citep{guo2025deepseek,aggarwal2025l,jiang2026dvao} as \emph{advantage priority inversion}.
Reducing the auxiliary weight can make such inversions less frequent, but
cannot structurally prevent them when the relative advantage magnitudes vary
across samples.
\textit{(ii) Hard priority mechanisms discard useful supervision.}
Reward-level gating~\citep{liu2026gdpo,liu2026learn,liu2025dler,microsoft2026mai} can prevent the auxiliary objective from contributing when
a primary condition is not satisfied, but does so through a coarse binary
decision and removes potentially informative distinctions among gated
responses.
\textit{(iii) Response-level aggregation cannot resolve token-level conflicts
in PPO.}
Actor--critic PPO estimates advantages at the token level through GAE, and the
primary and auxiliary objectives may agree at some token positions but conflict
at others within the same trajectory.
Popular response-level aggregation rules~\citep{microsoft2026mai,liu2026gdpo,liu2026learn} are designed only for GRPO and can therefore obscure the local conflicts
that directly determine token-level policy updates.

To address these challenges, we propose \textbf{Primary Advantage Policy
Optimization (PAPO)}, a priority-aware advantage aggregation framework with
three corresponding designs.
\textit{(i) PAPO preserves primary priority during advantage fusion.}
When the primary and auxiliary advantages agree, PAPO retains the auxiliary
contribution; when they conflict, PAPO smoothly bounds the auxiliary influence
relative to the primary advantage so that it cannot reverse the primary sign.
This sign-preserving fusion further induces a
\emph{primary-consistent ordering}: optimization units favored by the primary
objective remain ordered above primary-disfavored units.
Although normalization may change the absolute sign of aggregated advantage, it still preserves preference ordering.

\textit{(ii) PAPO retains fine-grained auxiliary supervision rather than discarding
it through hard gating.}
Instead of making a binary decision on whether the auxiliary signal is used,
PAPO continuously attenuates conflicting auxiliary advantages according to
their magnitude relative to the primary advantage.
This allows auxiliary supervision to remain informative while preventing it
from overturning the priority induced by the primary objective.

\textit{(iii) PAPO generalizes to the different policy
optimization methods.}
For GRPO-based optimization~\citep{guo2025deepseek}, PAPO fuses independently
normalized sequence-level primary and auxiliary advantages within each rollout
group.
For actor--critic PPO~\citep{schulman2017proximal}, PAPO instead performs
token-level fusion to resolve local primary--auxiliary conflicts within a
trajectory.
We further introduce a dual-head critic that separately estimates the primary
and auxiliary returns before nonlinear actor-side PAPO aggregation, making sure the accurate critic estimation.

Our contributions are summarized as follows:
\begin{itemize}
    \item We identify \emph{advantage priority inversion}, a problem in
    primary--auxiliary reward optimization where a conflicting auxiliary
    advantage can overturn the preference induced by the primary objective.
    We further highlight the limitations of existing weighting, hard-gating,
    and response-level aggregation strategies for preserving such priority.

    \item We propose \textbf{PAPO}, a priority-aware advantage aggregation
    method that asymmetrically combines primary and auxiliary advantages.
    PAPO preserves the primary sign during fusion while retaining fine-grained
    auxiliary supervision, and induces a \emph{primary-consistent ordering}
    that remains preserved under centered normalization.

    \item We instantiate PAPO at different optimization granularities:
    sequence-level aggregation for GRPO and token-level aggregation for PPO.
    For PPO, we further introduce a dual-head critic that preserves accurate
    objective-specific token-level credit assignment while keeping nonlinear
    PAPO fusion on the actor side.

    \item We evaluate PAPO on multiple domains and models
    across both GRPO- and PPO-based optimization, demonstrating consistent notable
    improvements on the primary objective while maintaining strong performance
    on the auxiliary objective.
\end{itemize}

 \section{Related Works}
\label{sec:related_work}
\textbf{Policy Optimization for LLM Reasoning.}
RL has become an effective paradigm for improving the reasoning capabilities of LLMs.
PPO~\citep{schulman2017proximal} and its variant VAPO~\citep{yue2025vapo} has been widely used for language-model post-training, while GRPO~\citep{guo2025deepseek} and its variants~\citep{zheng2025gspo,ma2026fipo} eliminate the learned critic by estimating relative advantages from groups of rollouts.
Recent method DAPO~\citep{yu2026dapo} further improves GRPO by introducing asymmetric clipping and dynamic sampling for large-scale training.

\textbf{Multi-Reward Reinforcement Learning.}
Multi-reward RL has been increasingly used to optimize multiple dimensions of LLM behavior.
Safe RLHF~\citep{dai2023saferlhf} separates helpfulness and safety into reward and cost objectives under constrained optimization. \cite{liu2026learn,liu2026gdpo,liu2025dler} condition the length reward on the accuracy reward to increase the efficiency for reasoning.
More recently, AW-GRPO~\citep{ichihara2025awgrpo} and DVAO~\citep{jiang2026dvao} dynamically adjust reward weights according to the learning progress of individual objectives, and~\cite{wang2026smopd} further studies imbalanced optimization among heterogeneous reward granularities.

\section{Preliminaries}

\textbf{Proximal Policy Optimization.}
For a token action $a_t$ at state $s_t$, PPO defines the importance ratio
$\rho_t(\theta)=
\frac{\pi_\theta(a_t\mid s_t)}
{\pi_{\theta_{\mathrm{old}}}(a_t\mid s_t)}$
and optimizes the clipped surrogate objective
\begin{equation}
L^{\mathrm{PPO}}(\theta)
=
\mathbb{E}_t
\left[
\min
\left(
\rho_t(\theta) A_t,
\mathrm{clip}(\rho_t(\theta),1-\epsilon_p,1+\epsilon_p)A_t
\right)
\right].
\end{equation}
In actor--critic PPO, the token-level advantage $A_t$ is commonly estimated
using Generalized Advantage Estimation (GAE). Given a value function $V_\phi$,
the TD error is
$\delta_t=r_t+\gamma V_\phi(s_{t+1})-V_\phi(s_t)$,
and the corresponding advantage is
$A_t=\sum_{l=0}^{T-t}(\gamma\lambda)^l\delta_{t+l}$.
The value function is trained toward the return target
$R_t=A_t+V_{\mathrm{old}}(s_t)$, typically using a clipped value loss for
stability.

\textbf{Group Relative Policy Optimization.}
GRPO is a sequence-level policy optimization method for response-level reward
feedback. Given a prompt $x$, GRPO samples a group of responses
$\{y_i\}_{i=1}^{G}$ from the old policy and obtains a reward $R_i$ for each
response. It estimates the relative advantage by normalizing rewards within the
group:
$A_i=
\frac{R_i-\mu_{\mathrm{group}}}
{\sigma_{\mathrm{group}}+\epsilon}$,
where $\mu_{\mathrm{group}}$ and $\sigma_{\mathrm{group}}$ denote the group
mean and standard deviation. Define $\rho_{i,t}(\theta)=
\frac{\pi_\theta(a_{i,t}\mid s_{i,t})}
{\pi_{\theta_{\mathrm{old}}}(a_{i,t}\mid s_{i,t})}$, the resulting response-level advantage is then
used in a PPO-style clipped objective:
\begin{equation}
L^{\mathrm{GRPO}}(\theta)
=
\mathbb{E}_{i,t}
\left[
\min
\left(
\rho_{i,t}(\theta) A_i,
\mathrm{clip}(\rho_{i,t}(\theta),1-\epsilon_p,1+\epsilon_p)A_i
\right)
\right].
\end{equation}

\section{Revisiting Multi-Reward LLM: Advantage Priority Inversion}
\label{sec:motivation}


\begin{figure}[t]
    \centering
    \vspace{-5mm}
    \includegraphics[width=0.75\linewidth]{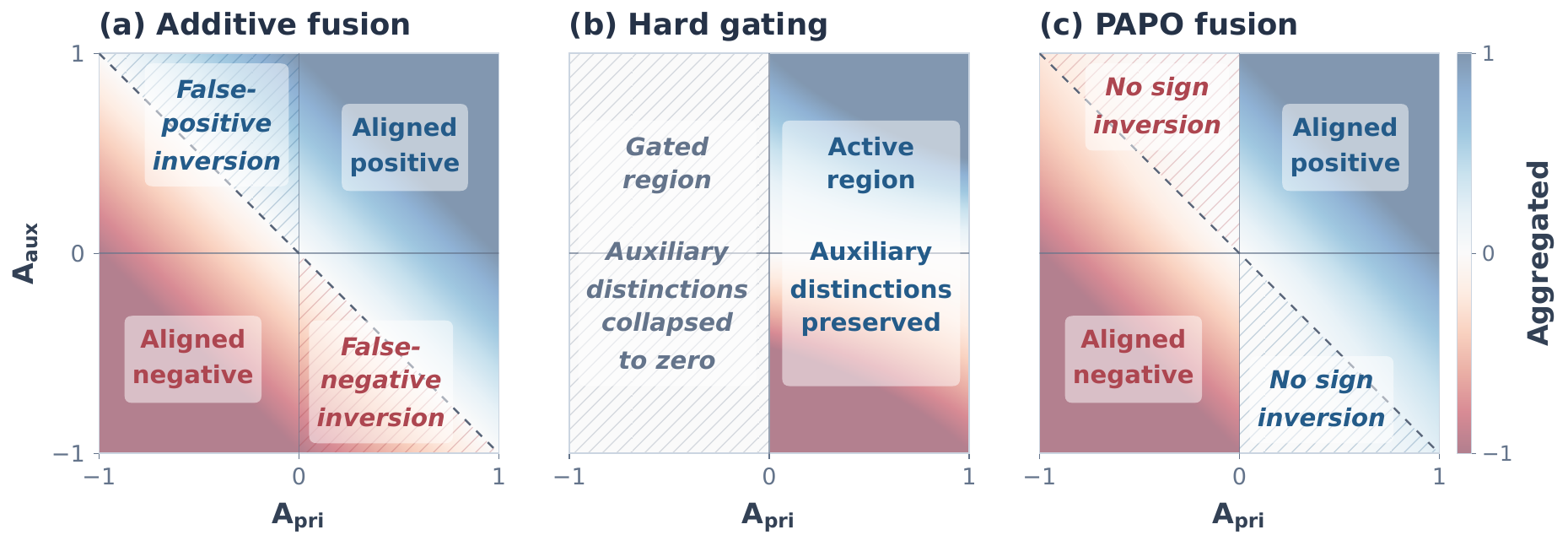}
    \vspace{-5mm}
\caption{
\textbf{Advantage aggregation.}
\textbf{(a)} Additive aggregation can cause \emph{advantage priority inversion},
where a dominant auxiliary signal reverses the sign induced by the primary advantage.
\textbf{(b)} Hard gating enforces priority coarsely by discarding auxiliary
distinctions for unsatisfied primary condition.
\textbf{(c)} PAPO preserves the primary sign during fusion while retaining
graded auxiliary supervision, maintaining the primary-induced
preference structure.
}
    \vspace{-2mm}
    \label{fig:motivation}
\end{figure}

Popular advantage aggregation methods mentioned in Section~\ref{sec:related_work} includes additive aggregation and gating, while these kind of methods have limitations on advantage priority inversion and losing informative auxiliary distinctions as illustrated in Figure~\ref{fig:motivation}, which motivates our PAPO design.

\textbf{Additive aggregations invert primary preference.}
The common additive approaches~\citep{guo2025deepseek,aggarwal2025l} sum primary advantage $A_{i,\mathrm{pri}}$ and auxiliary advantage $A_{i,\mathrm{aux}}$ for response $i$ in this form
:
$A_i^{\mathrm{sum}}
=
A_{i,\mathrm{pri}}
+
\alpha A_{i,\mathrm{aux}}$,
where $\alpha>0$ controls the auxiliary contribution.
As illustrated in Figure~\ref{fig:motivation}(a), when the two advantages
conflict and the auxiliary term dominates (i.e.,
$A_{i,\mathrm{pri}}A_{i,\mathrm{aux}}<0$ and
$\alpha|A_{i,\mathrm{aux}}|>|A_{i,\mathrm{pri}}|$),
the aggregated advantage satisfies
$\operatorname{sign}(A_i^{\mathrm{sum}})
\neq
\operatorname{sign}(A_{i,\mathrm{pri}})$.
We refer to this as \emph{advantage priority inversion}: such an inversion allows the auxiliary objective to overturn the preference
induced by the primary objective during aggregation. For example, a response
with $A_{i,\mathrm{pri}}<0$ can receive a positive aggregated advantage due to
a sufficiently large positive auxiliary signal, while a response with
$A_{i,\mathrm{pri}}>0$ can be suppressed by a sufficiently negative one.
Adjusting the auxiliary weight can reduce the frequency of such conflicts
\citep{liu2026gdpo,jiang2026dvao}, but does not structurally prevent it.

\textbf{Existing priority mechanisms remain coarse.}
Another strategy introduces priority through reward-level gating
\citep{liu2026gdpo,liu2026learn,liu2025dler,microsoft2026mai}:
$R_{i,\mathrm{aux}}
\gets
R_{i,\mathrm{aux}}
\Ind{R_{i,\mathrm{pri}}\ge\tau}$, where $\Ind{\cdot}$ is indicator function and $r$ is gating value.
Such gating prevents the auxiliary reward from contributing when the primary
condition is not satisfied by making a binary decision. However, as illustrated in Figure~\ref{fig:motivation}(b), responses that fail the
primary condition receive no auxiliary distinction even when their auxiliary
rewards can provide useful preference information.
This motivates a softer mechanism that retains auxiliary distinctions while
preventing advantage priority inversion.

\textbf{Sequence-level aggregation cannot resolve token-level conflicts.}
Advantage priority inversion becomes more pronounced in actor--critic PPO, where GAE provides a distinct advantage at each token.
The primary and auxiliary objectives may agree at some token positions but
conflict at others within the same sequence. Consequently, agreement between
sequence-level signals does not imply local agreement between token-level
advantages ($\hat{A}_t^{\mathrm{pri}}
\hat{A}_t^{\mathrm{aux}}
\ge 0$) for every token $t$, even when the two objectives appear to agree at the sequence level.
Previous sequence-level aggregation methods mentioned in Section~\ref{sec:related_work} therefore hide local conflicts with advantage priority inversion that
directly determine token-level policy updates, and can not be adapted to PPO.


\section{Method}
\label{sec:method}

In light of the limitations analyzed in
Section~\ref{sec:motivation}, we propose \textbf{Primary Advantage Policy
Optimization (PAPO)}, illustrated in Figure~\ref{fig:papo_method}.
PAPO performs priority-aware fusion directly in the advantage space: the
primary advantage determines the preferred update direction, while the
auxiliary advantage provides supplementary supervision without overriding this
priority. This design preserves the primary sign during advantage fusion and,
under standard centered normalization, maintains the induced
\emph{primary-consistent ordering}.
Because PAPO operates at the advantage level, the same principle can be
applied to different policy-optimization frameworks. We develop a
sequence-level variant for GRPO and a token-level variant for PPO, where a
multi-head critic provides accurate objective-specific credit assignment before the
advantages are aggregated by PAPO.

\subsection{Primary Advantage Policy Optimization (PAPO)}
\label{sec:papo}

Let $A_{\mathrm{pri}}$ and  $A_{\mathrm{aux}}$ denote the primary and auxiliary reward objectives, and PAPO combines them using a directional asymmetric
squashing rule:
\begin{equation}
A_{\mathrm{agg}}
=
A_{\mathrm{pri}}
+
\Ind{A_{\mathrm{pri}} A_{\mathrm{aux}} \ge 0}
\alpha A_{\mathrm{aux}}
+
\Ind{A_{\mathrm{pri}} A_{\mathrm{aux}} < 0}
A_{\mathrm{pri}}
\tanh\left(
\frac{
\alpha A_{\mathrm{aux}}
}{
A_{\mathrm{pri}}
+
\epsilon\,\operatorname{sgn}(A_{\mathrm{pri}})
}
\right),
\label{eq:papo}
\end{equation}
where $\alpha$ controls the strength of auxiliary supervision,
$\operatorname{sgn}(\cdot)$ denotes the sign function,
$\Ind{\cdot}$ is the indicator function, and $\epsilon$ is a small constant
for numerical stability. The aggregation depends on whether the primary and auxiliary objectives agree.
When $A_{\mathrm{pri}}$ and $A_{\mathrm{aux}}$ have the same sign, PAPO
retains the auxiliary signal because both objectives favor the same update
direction. When their signs conflict, Eq.~\ref{eq:papo} can be written as $A_{\mathrm{agg}}
=
A_{\mathrm{pri}}
\left[
1-
\tanh\left(
\frac{
\alpha |A_{\mathrm{aux}}|
}{
|A_{\mathrm{pri}}|+\epsilon
}
\right)
\right].$
The multiplicative factor is strictly positive and at most one. Thus, a
conflicting auxiliary signal can smoothly attenuate the primary advantage
toward zero, but cannot reverse its sign. The $\tanh$ transformation further
provides a smooth transition from approximately linear aggregation under weak
conflict to saturation under strong conflict. Moreover, this sign preservation is the key mechanism by which PAPO enforces objective
priority. In particular, \textit{it induces an \textbf{ordering} between optimization
units favored and disfavored by the primary objective.} This ordering preserves the coarse preference induced by the primary objective:
auxiliary supervision may refine preferences within each side of the primary
decision boundary, but cannot promote a primary-disfavored unit above a
primary-favored one. The following
Proposition~\ref{prop:papo_priority_main} shows that this primary-consistent ordering is preserved even
after the centered normalization used in GRPO and PPO.

\begin{propositionbox}
\begin{proposition}
\label{prop:papo_priority_main}
Let an optimization unit denote a rollout in GRPO or a token position in PPO;
the associated normalization set is a rollout group for GRPO and a token-level
mini-batch for PPO.
For any unit $u$ with $A_{u,\mathrm{pri}}\neq 0$, PAPO preserves the primary
sign during aggregation:
$\operatorname{sign}(A_{u,\mathrm{agg}})
=
\operatorname{sign}(A_{u,\mathrm{pri}})$.
Moreover, for any units $u$ and $v$ in the same normalization,
centered normalization
$\hat{A}_k=(A_{k,\mathrm{agg}}-\mu)/(\sigma+\epsilon)$
preserves their primary-consistent ordering:
\[
A_{u,\mathrm{pri}}>0>A_{v,\mathrm{pri}}
\quad\Longrightarrow\quad
\hat{A}_u>\hat{A}_v.
\]
\end{proposition}
\end{propositionbox}

\textit{Proof sketch.}
For conflicting objectives, Eq.~\ref{eq:papo} expresses
$A_{\mathrm{agg}}$ as $A_{\mathrm{pri}}$ multiplied by
$1-\tanh(\alpha|A_{\mathrm{aux}}|/(|A_{\mathrm{pri}}|+\epsilon))$,
which is strictly positive. Hence, PAPO preserves the primary sign during
aggregation, so
$A_{u,\mathrm{pri}}>0>A_{v,\mathrm{pri}}$ implies
$A_{u,\mathrm{agg}}>0>A_{v,\mathrm{agg}}$.
Centered normalization is a strictly increasing affine transformation because
$\sigma+\epsilon>0$, and therefore preserves this ordering.
Full proofs are provided in Appendix~\ref{app:normalization_priority}.
We next instantiate this aggregation rule in GRPO and PPO, with their corresponding advantage estimation and normalization procedures.

\subsection{GRPO-based PAPO: Sequence-Level Advantage Aggregation}
\label{sec:papo_grpo}

For GRPO-based methods, PAPO operates on sequence-level advantages computed
for each rollout. Given the primary and auxiliary rewards of the $i$-th rollout, we first
form their corresponding raw advantages,
$A_{i,\mathrm{pri}}$ and $A_{i,\mathrm{aux}}$.
Because the two reward signals may have different scales, we normalize them
independently within each rollout group:
$A_{i,k}
=
\frac{
A_{i,k}-\mu_{\mathrm{group}}^{k}
}{
\sigma_{\mathrm{group}}^{k}+\epsilon
},
k\in\{\mathrm{pri},\mathrm{aux}\}.$ The normalized advantages are then fused using Eq.~\ref{eq:papo} to obtain
$A_i^{\mathrm{agg}}$. To retain the conventional zero-mean advantage baseline,
we further apply centered normalization within the same rollout group:
$\hat{A}_i
=
\frac{
A_i^{\mathrm{agg}}
-
\mu_{\mathrm{group}}^{\mathrm{agg}}
}{
\sigma_{\mathrm{group}}^{\mathrm{agg}}
+\epsilon
}.$
As established in Proposition~\ref{prop:papo_priority_main}, PAPO preserves
the primary sign during fusion, while the subsequent centered normalization
preserves the resulting \emph{primary-consistent ordering}. Therefore, although
the absolute sign of an individual normalized advantage may change after
centering, a response disfavored by the primary objective cannot be ranked
above a primary-favored response due to the auxiliary objective.

\begin{figure*}[t]
    \centering
    \includegraphics[width=\textwidth]{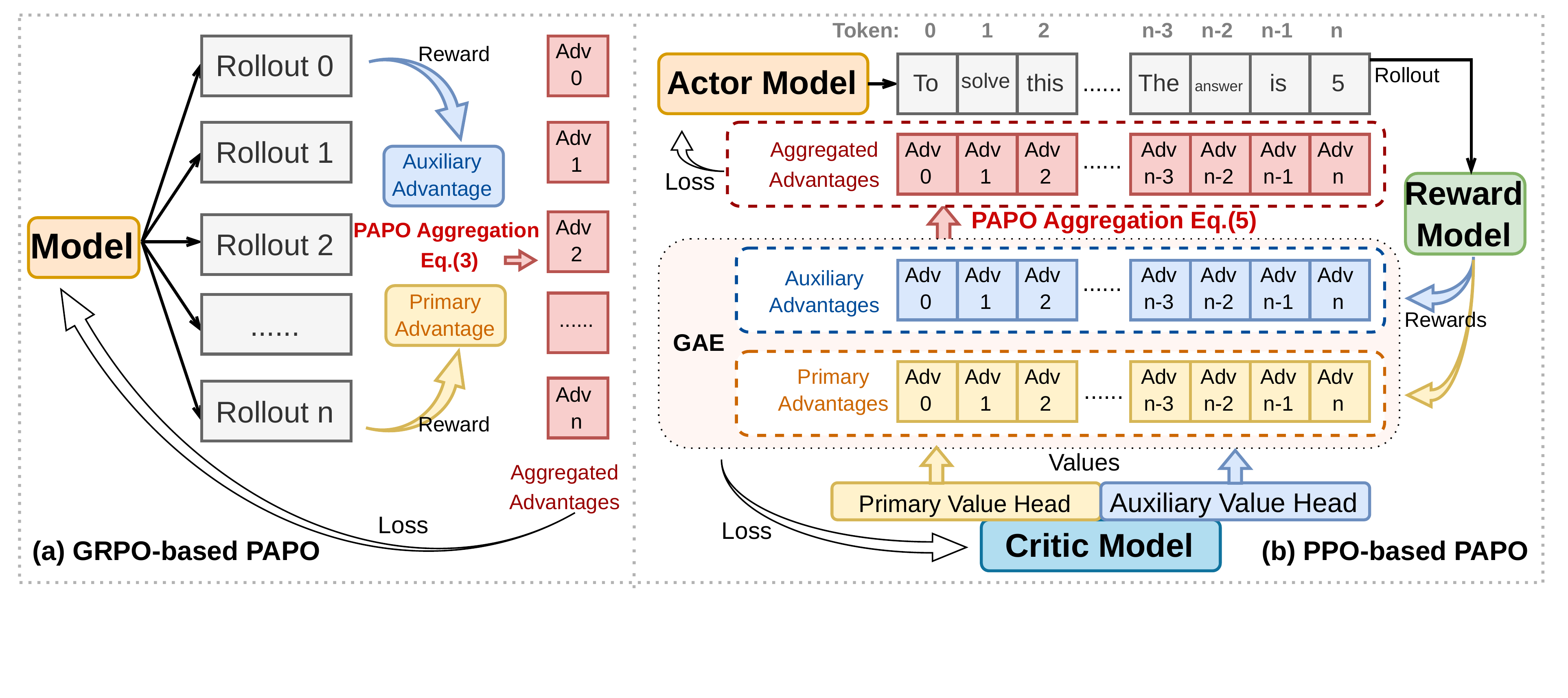}
    \vspace{-1.6cm}
    \caption{
    \textbf{Overview.}
    \textbf{(a) GRPO-based PAPO.}
    Given multiple rollouts, primary and auxiliary rewards are converted
    to sequence-level advantages and combined.
    \textbf{(b) PPO-based PAPO.}
    A multi-head critic estimates objective-specific values for the primary
    and auxiliary objectives. Token-level advantages are then computed using
    GAE and fused by PAPO before optimizing the actor.
    }
    \vspace{-5mm}
    \label{fig:papo_method}
\end{figure*}

\subsection{PPO-based PAPO: Token-Level Advantage Aggregation}
\label{sec:papo_ppo}

\paragraph{Token-Level Advantage with PAPO.}

Actor--critic PPO estimates a distinct advantage at each token through GAE.
This token-level granularity is important for primary--auxiliary optimization:
even within the same trajectory, the primary and auxiliary objectives may agree
at some token positions but conflict at others. A response-level aggregation
can therefore obscure local conflicts that directly determine the policy
update. Therefore for PPO, we apply PAPO at the token level and estimate the two
objectives separately before fusion. Specifically, we maintain two value heads,
$V_\phi^{\mathrm{pri}}(s_t)$ and $V_\phi^{\mathrm{aux}}(s_t)$,
for the primary and auxiliary objectives, respectively. For
$k\in\{\mathrm{pri},\mathrm{aux}\}$, we independently compute the TD error
and GAE:
\begin{equation}
\delta_t^{k}
=
r_t^{k}
+
\gamma V_\phi^{k}(s_{t+1})
-
V_\phi^{k}(s_t),
\qquad
\hat{A}_t^{k}
=
\sum_{l=0}^{T-t}
(\gamma\lambda)^l
\delta_{t+l}^{k}.
\label{eq:papo_gae}
\end{equation}
Computing the two advantages independently preserves objective-specific
temporal credit assignment before PAPO fusion. We then apply PAPO directly to
$\hat{A}_t^{\mathrm{pri}}$ and $\hat{A}_t^{\mathrm{aux}}$ at every token:
\begin{equation}
\hat{A}_t^{\mathrm{agg}}
=
\hat{A}_t^{\mathrm{pri}}
+
\Ind{
\hat{A}_t^{\mathrm{pri}}
\hat{A}_t^{\mathrm{aux}}
\ge 0
}
\alpha \hat{A}_t^{\mathrm{aux}}
+
\Ind{
\hat{A}_t^{\mathrm{pri}}
\hat{A}_t^{\mathrm{aux}}
<0
}
\hat{A}_t^{\mathrm{pri}}
\tanh\left(
\frac{
\alpha\hat{A}_t^{\mathrm{aux}}
}{
\hat{A}_t^{\mathrm{pri}}
+
\epsilon\,\operatorname{sgn}(
\hat{A}_t^{\mathrm{pri}}
)
}
\right).
\label{eq:papo_ppo}
\end{equation}

Then the fused token-level advantages are centered and
normalized within the optimization mini-batch:
$\hat{A}_t
=
\frac{
\hat{A}_t^{\mathrm{agg}}
-
\mu_{\mathrm{batch}}^{\mathrm{agg}}
}{
\sigma_{\mathrm{batch}}^{\mathrm{agg}}
+
\epsilon
}.$
As established in Proposition~\ref{prop:papo_priority_main}, PAPO preserves
the primary sign at each token during fusion, while the subsequent centered
normalization preserves the resulting \emph{primary-consistent ordering}
within the normalization set. Thus, although centering may change the absolute
sign of an individual token advantage, an auxiliary signal cannot rank a
primary-disfavored token position above a primary-favored one. This allows
PAPO to modulate local conflicts with fine-grained auxiliary supervision while
retaining the priority structure induced by the primary objective.

\paragraph{Dual-Head Critic Training.}

If we directly train a critic according to the fused advantage in
Eq.~\ref{eq:papo_ppo} like actor, it would
couple value estimation to the PAPO aggregation rule and collapse information from the primary and auxiliary objectives into a single target, because the fusion is generally many-to-one. Instead, we design two critic heads to independently
estimate the returns of the primary and auxiliary objectives, while PAPO
determines how their resulting advantages interact only during the actor
update.
For each objective $k\in\{\mathrm{pri},\mathrm{aux}\}$, we construct its
return target from the corresponding GAE estimate:
\begin{equation}
R_t^{k}
=
\hat{A}_t^{k}
+
V_{\mathrm{old}}^{k}(s_t),
\qquad
k\in\{\mathrm{pri},\mathrm{aux}\}.
\label{eq:papo_return}
\end{equation}
Each value head is therefore trained on its own unfused target, preserving the
semantics and token-level credit assignment of the corresponding reward
objective independently of the actor-side PAPO transformation.
Each head is optimized using the standard clipped value loss:
\begin{equation}
L_{k}^{\mathrm{VF}}(\phi)
=
\frac{1}{B}
\sum_{\mathrm{batch}}
\max
\left[
\left(
V_\phi^{k}(s_t)-R_t^{k}
\right)^2,
\left(
\mathrm{clip}
\left(
V_\phi^{k}(s_t),
V_{\mathrm{old}}^{k}(s_t)-\epsilon_c,
V_{\mathrm{old}}^{k}(s_t)+\epsilon_c
\right)
-
R_t^{k}
\right)^2
\right],
\label{eq:papo_value_loss}
\end{equation}
where $\epsilon_c$ denotes clipping range. The overall critic
objective is
$L^{\mathrm{VF}}(\phi)
=
L_{\mathrm{pri}}^{\mathrm{VF}}(\phi)
+
L_{\mathrm{aux}}^{\mathrm{VF}}(\phi).$

Therefore, this yields a clear division of responsibility: the dual-head critic provides
separate token-level credit assignment for the primary and auxiliary
objectives, whereas PAPO advantage aggregation performs priority-aware nonlinear fusion only at the
actor update. Consequently, value estimation remains objective-specific and
independent of the PAPO aggregation rule.

\section{Experiments}

We design our experiments around four research questions:
whether PAPO mainly improves the primary objective while having competitive auxiliary
behavior (\textbf{RQ1}), how priority-aware aggregation changes the optimization
dynamics of competing objectives (\textbf{RQ2}), which design components are responsible for the improvement
(\textbf{RQ3}), and whether the resulting gains generalize across models, scales, reward semantics, and domains (\textbf{RQ4}).
To answer \textbf{RQ1}, we evaluate PAPO on two representative multi-reward settings, competitional math reasoning and tool use.

\textbf{Model and Baselines.} We mainly include the \textbf{Base} model Qwen3-8B \citep{yang2025qwen3}and standard \textbf{PPO}~\citep{schulman2017proximal} and
\textbf{GRPO}~\citep{guo2025deepseek} with additive multi-reward aggregation, covering token- and
sequence-level policy optimization. We additionally compare with the representative gating method \textbf{GDPO}~\citep{liu2026gdpo},
which separately normalizes rewards. Our PAPO shares the same settings as the corresponding PPO or GRPO baseline. We also include additional models and baselines in Section~\ref{sec:additional_models} and Appendix~\ref{app:exp_results} to evaluate the effectiveness of PAPO.

\begin{table}[t]
\centering
\vspace{-3mm}
\caption{
\textbf{Mathematical reasoning results.}
$\Delta$ rows report changes relative to the Base model.
Higher is better for both accuracy (Acc.$\uparrow$) and length compliance (Len.$\uparrow$).
The best result within each optimization family is in \textbf{bold}.
}
\vspace{1mm}
\label{tab:main_results_transposed}

\tablesize
\setlength{\tabcolsep}{6.0pt}
\renewcommand{\arraystretch}{1.22}

\begin{adjustbox}{max width=\columnwidth,center}
\begin{tabular}{@{}lcc|cccc|cccccc@{}}

\toprule

&
\multicolumn{2}{c|}{\textbf{Base}}
& \multicolumn{4}{c|}{\textbf{PPO-based}}
& \multicolumn{6}{c}{\textbf{GRPO-based}}
\\

\cmidrule(lr){2-3}
\cmidrule(lr){4-7}
\cmidrule(lr){8-13}

\textbf{Benchmark}
& \multicolumn{2}{c|}{\textbf{Qwen3-8B}}
& \multicolumn{2}{c}{\textbf{PPO}}
& \multicolumn{2}{c|}{\textbf{PAPO}}
& \multicolumn{2}{c}{\textbf{GRPO}}
& \multicolumn{2}{c}{\textbf{GDPO}}
& \multicolumn{2}{c}{\textbf{PAPO}}
\\

\cmidrule(lr){2-3}
\cmidrule(lr){4-5}
\cmidrule(lr){6-7}
\cmidrule(lr){8-9}
\cmidrule(lr){10-11}
\cmidrule(lr){12-13}

&
Acc. & Len.
& Acc. & Len.
& Acc. & Len.
& Acc. & Len.
& Acc. & Len.
& Acc. & Len.
\\

\midrule

\textbf{Olympiad}
& 38.8 & 51.6
& 35.9 & 77.4
& \textbf{42.9} & \textbf{87.2}
& 38.9 & 75.7
& 39.6 & 74.1
& \textbf{43.2} & \textbf{90.2}
\\

\rowcolor{deltagray}
\dname{Olympiad}
& \ddash & \ddash
& \dneg{-2.9} & \dpos{25.8}
& \dposbf{4.1} & \dposbf{35.6}
& \dpos{0.1} & \dpos{24.1}
& \dpos{0.8} & \dpos{22.5}
& \dposbf{4.4} & \dposbf{38.6}
\\

\addlinespace[1.5pt]

\textbf{Minerva}
& 16.7 & 95.8
& 16.7 & \textbf{97.9}
& \textbf{18.8} & \textbf{97.9}
& 16.7 & 97.9
& 18.8 & 97.9
& \textbf{20.8} & \textbf{100.0}
\\

\rowcolor{deltagray}
\dname{Minerva}
& \ddash & \ddash
& \dzero & \dposbf{2.1}
& \dposbf{2.1} & \dposbf{2.1}
& \dzero & \dpos{2.1}
& \dpos{2.1} & \dpos{2.1}
& \dposbf{4.1} & \dposbf{4.2}
\\

\addlinespace[1.5pt]

\textbf{AMC}
& 82.5 & 60.0
& 77.5 & \textbf{92.5}
& \textbf{95.0} & 85.0
& 82.5 & \textbf{97.5}
& 85.0 & 87.5
& \textbf{92.5} & 87.5
\\

\rowcolor{deltagray}
\dname{AMC}
& \ddash & \ddash
& \dneg{-5.0} & \dposbf{32.5}
& \dposbf{12.5} & \dpos{25.0}
& \dzero & \dposbf{37.5}
& \dpos{2.5} & \dpos{27.5}
& \dposbf{10.0} & \dpos{27.5}
\\

\addlinespace[1.5pt]

\textbf{MATH}
& 85.4 & 83.0
& 85.0 & \textbf{97.0}
& \textbf{89.6} & 94.8
& 84.4 & 92.8
& 85.2 & 92.6
& \textbf{89.8} & \textbf{97.8}
\\

\rowcolor{deltagray}
\dname{MATH}
& \ddash & \ddash
& \dneg{-0.4} & \dposbf{14.0}
& \dposbf{4.2} & \dpos{11.8}
& \dneg{-1.0} & \dpos{9.8}
& \dneg{-0.2} & \dpos{9.6}
& \dposbf{4.4} & \dposbf{14.8}
\\

\addlinespace[1.5pt]

\textbf{AIME24}
& 56.7 & 33.3
& 23.3 & 43.3
& \textbf{53.3} & \textbf{70.0}
& 46.7 & 63.3
& 43.3 & 60.0
& \textbf{70.0} & \textbf{76.7}
\\

\rowcolor{deltagray}
\dname{AIME24}
& \ddash & \ddash
& \dneg{-33.4} & \dpos{10.0}
& \dnegbf{-3.4} & \dposbf{36.7}
& \dneg{-10.0} & \dpos{30.0}
& \dneg{-13.4} & \dpos{26.7}
& \dposbf{13.3} & \dposbf{43.4}
\\

\addlinespace[1.5pt]

\textbf{AIME25}
& 40.0 & 16.7
& 23.3 & \textbf{53.3}
& \textbf{40.0} & 46.7
& 36.7 & 60.0
& 33.3 & 53.3
& \textbf{43.3} & \textbf{73.3}
\\

\rowcolor{deltagray}
\dname{AIME25}
& \ddash & \ddash
& \dneg{-16.7} & \dposbf{36.6}
& \dzerobf & \dpos{30.0}
& \dneg{-3.3} & \dpos{43.3}
& \dneg{-6.7} & \dpos{36.6}
& \dposbf{3.3} & \dposbf{56.6}
\\

\midrule



\textbf{Average}
& 53.4 & 56.7
& 43.6 & 76.9
& \textbf{56.6} & \textbf{80.3}
& 51.0 & 81.2
& 50.9 & 77.6
& \textbf{59.9} & \textbf{87.6}
\\

\rowcolor{deltagray}
\dname{Average}
& \ddash & \ddash
& \dneg{-9.8} & \dpos{20.2}
& \dposbf{3.2} & \dposbf{23.6}
& \dneg{-2.4} & \dpos{24.5}
& \dneg{-2.5} & \dpos{20.9}
& \dposbf{6.5} & \dposbf{30.9}
\\

\bottomrule

\end{tabular}
\end{adjustbox}

\vspace{-5mm}
\end{table}

\begin{table}[t]
\centering
\caption{
\textbf{Agentic tool-use results on} \textsc{When2Tool}~\citep{sun2026when2tool}.
$\Delta$ rows report changes relative to the Base model.
Acc.$\uparrow$ denotes accuracy and Calls$\downarrow$ denotes the average number of tool calls for a query.
The best result within each optimization family is in \textbf{bold}.
}
\vspace{1mm}
\label{tab:tool_use_results}

\tablesize
\setlength{\tabcolsep}{6.5pt}
\renewcommand{\arraystretch}{1.22}

\begin{adjustbox}{max width=\columnwidth,center}
\begin{tabular}{@{}lcc|cccc|cccccc@{}}

\toprule

&
\multicolumn{2}{c|}{\textbf{Base}}
& \multicolumn{4}{c|}{\textbf{PPO-based}}
& \multicolumn{6}{c}{\textbf{GRPO-based}}
\\

\cmidrule(lr){2-3}
\cmidrule(lr){4-7}
\cmidrule(lr){8-13}

\textbf{Difficulty}
& \multicolumn{2}{c|}{\textbf{Qwen3-8B}}
& \multicolumn{2}{c}{\textbf{PPO}}
& \multicolumn{2}{c|}{\textbf{PAPO}}
& \multicolumn{2}{c}{\textbf{GRPO}}
& \multicolumn{2}{c}{\textbf{GDPO}}
& \multicolumn{2}{c}{\textbf{PAPO}}
\\

\cmidrule(lr){2-3}
\cmidrule(lr){4-5}
\cmidrule(lr){6-7}
\cmidrule(lr){8-9}
\cmidrule(lr){10-11}
\cmidrule(lr){12-13}

&
Acc. & Calls
& Acc. & Calls
& Acc. & Calls
& Acc. & Calls
& Acc. & Calls
& Acc. & Calls
\\

\midrule

\textbf{Easy}
& 88.11 & 1.33
& 93.78 & \textbf{0.15}
& \textbf{94.56} & 0.22
& 93.89 & 0.13
& 93.89 & \textbf{0.08}
& \textbf{95.67} & 0.12
\\

\rowcolor{deltagray}
\dname{Easy}
& \ddash & \ddash
& \dgood{+5.67} & \dgoodbf{-1.18}
& \dgoodbf{+6.45} & \dgood{-1.11}
& \dgood{+5.78} & \dgood{-1.20}
& \dgood{+5.78} & \dgoodbf{-1.25}
& \dgoodbf{+7.56} & \dgood{-1.21}
\\

\addlinespace[1.5pt]

\textbf{Medium}
& 88.89 & 1.41
& 87.22 & \textbf{0.23}
& \textbf{92.44} & 0.49
& 87.00 & 0.16
& \textbf{89.77} & \textbf{0.14}
& 89.22 & 0.15
\\

\rowcolor{deltagray}
\dname{Medium}
& \ddash & \ddash
& \dbad{-1.67} & \dgoodbf{-1.18}
& \dgoodbf{+3.55} & \dgood{-0.92}
& \dbad{-1.89} & \dgood{-1.25}
& \dgoodbf{+0.88} & \dgoodbf{-1.27}
& \dgood{+0.33} & \dgood{-1.26}
\\

\addlinespace[1.5pt]

\textbf{Hard}
& 75.67 & 1.47
& 77.78 & \textbf{0.73}
& \textbf{79.56} & 0.96
& 82.22 & 0.69
& 76.89 & \textbf{0.52}
& \textbf{84.89} & 0.71
\\

\rowcolor{deltagray}
\dname{Hard}
& \ddash & \ddash
& \dgood{+2.11} & \dgoodbf{-0.74}
& \dgoodbf{+3.89} & \dgood{-0.51}
& \dgood{+6.55} & \dgood{-0.78}
& \dgood{+1.22} & \dgoodbf{-0.95}
& \dgoodbf{+9.22} & \dgood{-0.76}
\\

\midrule



\textbf{Average}
& 84.22 & 1.40
& 86.26 & \textbf{0.37}
& \textbf{88.85} & 0.56
& 87.70 & 0.33
& 86.85 & \textbf{0.25}
& \textbf{89.93} & 0.33
\\

\rowcolor{deltagray}
\dname{Average}
& \ddash & \ddash
& \dgood{+2.04} & \dgoodbf{-1.03}
& \dgoodbf{+4.63} & \dgood{-0.84}
& \dgood{+3.48} & \dgood{-1.07}
& \dgood{+2.63} & \dgoodbf{-1.15}
& \dgoodbf{+5.71} & \dgood{-1.07}
\\

\bottomrule

\end{tabular}
\end{adjustbox}

\vspace{-4mm}
\end{table}

\subsection{Competition-level Math Reasonsing}

\textbf{Data and Reward Settings.}
We evaluate PAPO on competition-level mathematical reasoning with two objectives:
answer correctness as the \emph{primary} reward and reasoning efficiency as the
\emph{auxiliary} reward. We train on DAPO-17K\citep{yu2026dapo} and evaluate on OlympiadBench\citep{he2024olympiadbench},
Minerva Math\citep{lewkowycz2022solving}, AMC\citep{li2024numinamath}, MATH-500\citep{hendrycks2021measuring}, AIME 2024, and AIME 2025. For a response $y$, the
correctness reward is $R_{\mathrm{acc}}(y)=\Ind{\hat{y}=y^*}$, where $\hat{y}$ is the extracted final answer and
$y^*$ is the reference answer. We define the length compliance reward as
$R_{\mathrm{len}}(y)=\Ind{L(y)\leq T}$, where $L(y)$ is the number of generated
response tokens and $T=8{,}000$. We set $R_{\mathrm{pri}}=R_{\mathrm{acc}}$ and
$R_{\mathrm{aux}}=R_{\mathrm{len}}$. Finally, we report answer
accuracy and length-budget compliance, with accuracy serving as the primary
evaluation metric.

\textbf{Results.}
Table~\ref{tab:main_results_transposed} summarizes the results on mathematical
reasoning.
PAPO consistently improves the primary accuracy objective to the best under both PPO- and
GRPO-based training. On average, PAPO raises accuracy from 43.6 to 56.6 over
PPO and from 51.0 to 59.9 over GRPO. These gains are accompanied by the best
average length compliance, increasing from 76.9 to 80.3 and from 81.2 to 87.6,
respectively. Compared with GDPO, PAPO-GRPO improves average accuracy from
50.9 to 59.9 while also increasing length compliance from 77.6 to 87.6.
The accuracy gains are particularly pronounced on the more challenging
Olympiad and AIME benchmarks, where auxiliary efficiency objectives can more
strongly conflict with correctness. 

\subsection{Agentic Tool Using}
\label{sec:exp_tool}
\textbf{Data and Reward Settings.}
We evaluate PAPO on \textsc{When2Tool}~\citep{sun2026when2tool}, which contains 18 tool-use environments spanning single-hop and multi-hop tasks with easy, medium, and hard subsets. We treat answer correctness as the \emph{primary} reward and tool-call efficiency as the \emph{auxiliary} reward. For a response $y$, the correctness reward is $R_{\mathrm{acc}}(y)=\Ind{\hat{y}=y^*}$, where $\hat{y}$ is the extracted final answer and
$y^*$ is the reference answer. Let $N_{\mathrm{call}}$ denote the number of executed tool calls and $B$ the task-specific tool budget; we define the auxiliary reward as $R_{\mathrm{aux}}=\max(0,1-N_{\mathrm{call}}/(B+\delta))$ with $\delta=1$. The tool call budget $B$ is from task annotations, which use the number of annotated steps. Specifically, $B=3$ for multi-step examples, and $B=1$ for single-step examples. We report accuracy and the average number of tool calls, both average and across difficulty levels.

\textbf{Results.}
Table~\ref{tab:tool_use_results} reports the tool-use results, where PAPO improves average
accuracy from 86.26 to 88.85 over PPO and from 87.70 to 89.93 over GRPO.
Under GRPO, this gain is achieved with essentially unchanged tool usage
(0.33 calls on average), while under PPO, PAPO uses moderately more tools
(0.56 vs.\ 0.37) in exchange for higher task accuracy. Compared with GDPO,
PAPO-GRPO further raises average accuracy from 86.85 to 89.93, with only a
small increase in tool calls from 0.25 to 0.33. The improvement is especially
clear on hard subsets, where PAPO-GRPO increases accuracy from 82.22 to
84.89 while keeping tool usage nearly unchanged (0.69 to 0.71). Overall, PAPO first benefits task success rather than indiscriminately
maximizing the auxiliary objective.

\subsection{Training and Evaluation Dynamics}
\label{sec:dynamics}
\begin{figure}[t]
    \centering
    \includegraphics[width=\linewidth]{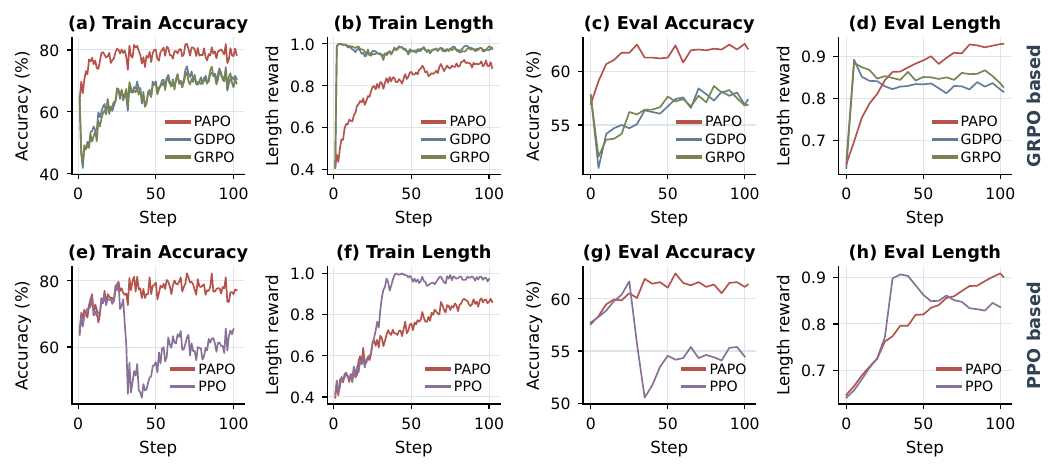}
    \vspace{-9mm}
    \caption{
    \textbf{PAPO Dynamics.}
    PAPO maintains the best accuracy $\uparrow$ while having a competitive auxiliary length compliance $\uparrow$ for GRPO-based (top) and PPO-based (bottom) methods.
    }
    \vspace{-5mm}
    \label{fig:math_dynamics}
\end{figure}

To answer \textbf{RQ2}, we further illustrate how PAPO balances correctness and reasoning efficiency throughout training on math data in Figure~\ref{fig:math_dynamics}. Across both GRPO- and PPO-based settings, baseline methods often optimize the auxiliary length objective rapidly without consistently improving evaluation accuracy, and this effect is particularly pronounced for PPO, where aggressive length optimization coincides with unstable or degraded correctness. In contrast, PAPO maintains stronger and more stable accuracy while progressively improving the length reward, ultimately achieving a better balance between the two objectives. These dynamics support our motivation that an easily optimized auxiliary objective can interfere with the primary task under unconstrained aggregation, whereas PAPO allows auxiliary supervision to refine the policy without overriding correctness.

\subsection{Ablation Studies}

To answer \textbf{RQ3}, we conduct ablations to isolate the contribution of each component.
\textbf{(a) Ablation on $\alpha$.}
Figure~\ref{fig:alpha_tradeoff} studies the effect of the auxiliary coefficient
$\alpha$. Simply tuning $\alpha$ under conventional weighted aggregation does
not provide reliable control over objective priority, as optimization can still
favor the objective that is easier to improve rather than following the
intended weighting \citep{liu2026gdpo}. In contrast, PAPO consistently maintains stronger
primary accuracy across different values of $\alpha$, while retaining
competitive auxiliary performance. \textbf{(b) Ablation on Multi-Head Critic.}
Table~\ref{tab:ppo_ablation} disentangles the contributions of the multi-head
critic (MHC) and PAPO aggregation under PPO-based training.
Introducing MHC alone provides modest improvements, increasing math accuracy
from 54.5 to 55.7 and tool-use accuracy from 86.25 to 86.84, together with
improved length compliance.
Adding PAPO on MHC yields larger gains, raising math
accuracy to 61.1 and tool-use accuracy to 88.25, indicating that both components contribute to PAPO's improvement.

\begin{figure*}[t]
\centering

\begin{minipage}[t]{0.48\textwidth}
\vspace{-4mm}
\centering
\captionof{table}{
PPO-based PAPO component ablation.
MHC denotes the multi-head critic. Best results are in
\textbf{bold} and second-best results are \underline{underlined}.
}
\vspace{1mm}
\small

\resizebox{\linewidth}{!}{%
\begin{tabular}{lcc|cc}
\toprule
& \multicolumn{2}{c|}{\textbf{Math}}
& \multicolumn{2}{c}{\textbf{Tool Use}} \\
\cmidrule(lr){2-3}
\cmidrule(lr){4-5}
\textbf{Method}
& \textbf{Acc.} $\uparrow$
& \textbf{Len.} $\uparrow$
& \textbf{Acc.} $\uparrow$
& \textbf{Calls} $\downarrow$ \\
\midrule

Base
& \underline{57.6}
& 64.1
& 84.19
& 1.40 \\

PPO
& 54.5
& 83.6
& 86.25
& \textbf{0.37} \\

PPO + MHC
& 55.7
& \underline{87.2}
& \underline{86.84}
& \underline{0.41} \\

PPO + MHC + PAPO
& \textbf{61.1}
& \textbf{90.2}
& \textbf{88.25}
& 0.56 \\

\bottomrule
\end{tabular}%
}

\vspace{2mm}

\label{tab:ppo_ablation}
\end{minipage}
\hfill
\begin{minipage}[t]{0.49\textwidth}
\vspace{-3mm}
\centering

\includegraphics[
    width=\linewidth,
    trim=0 5 0 2,
    clip
]{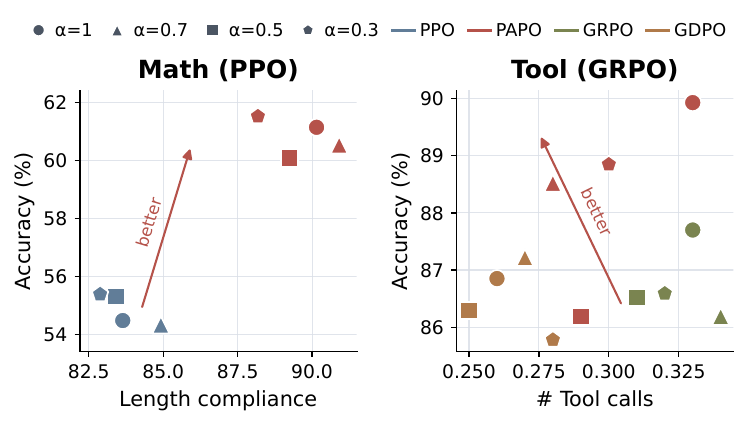}

\vspace{-4mm}
\captionof{figure}{
Sensitivity to the auxiliary weight $\alpha$.}

\label{fig:alpha_tradeoff}
\end{minipage}

\vspace{-5mm}
\end{figure*}


\begin{table*}[t]
\centering
\caption{
\textbf{Math results on DeepSeek-R1-Distill-1.5B.}
Higher is better for both accuracy (Acc.) and length compliance (Len.).
Best results are in \textbf{bold} and second-best results are
\underline{underlined}.
}
\vspace{1mm}
\label{tab:additional_math_model}

\small
\setlength{\tabcolsep}{5.0pt}
\renewcommand{\arraystretch}{1.08}

\begin{adjustbox}{max width=\textwidth,center}
\begin{tabular}{lcc|cc|cc|cc|cc|cc|cc}
\toprule
& \multicolumn{2}{c|}{\textbf{Olympiad}}
& \multicolumn{2}{c|}{\textbf{Minerva}}
& \multicolumn{2}{c|}{\textbf{AMC}}
& \multicolumn{2}{c|}{\textbf{MATH}}
& \multicolumn{2}{c|}{\textbf{AIME24}}
& \multicolumn{2}{c|}{\textbf{AIME25}}
& \multicolumn{2}{c}{\textbf{Overall}} \\
\cmidrule(lr){2-3}
\cmidrule(lr){4-5}
\cmidrule(lr){6-7}
\cmidrule(lr){8-9}
\cmidrule(lr){10-11}
\cmidrule(lr){12-13}
\cmidrule(lr){14-15}

\textbf{Method}
& Acc. & Len.
& Acc. & Len.
& Acc. & Len.
& Acc. & Len.
& Acc. & Len.
& Acc. & Len.
& Acc. & Len. \\
\midrule

Base
& 22.8 & 43.1
& 15.6 & 81.3
& 47.5 & 50.0
& 66.2 & 75.8
& 16.7 & 16.7
& 16.7 & 13.3
& 39.4 & 60.0 \\

PPO
& \underline{28.6} & \underline{66.1}
& \textbf{16.7} & \underline{93.8}
& \textbf{65.0} & \textbf{80.0}
& \underline{73.6} & \underline{91.0}
& \textbf{16.7} & \underline{36.7}
& \textbf{20.0} & \textbf{46.7}
& \underline{45.8} & \underline{75.8} \\

PAPO
& \textbf{29.9} & \textbf{74.6}
& \textbf{16.7} & \textbf{95.8}
& \underline{62.5} & \textbf{80.0}
& \textbf{78.0} & \textbf{91.8}
& \textbf{16.7} & \textbf{43.3}
& \textbf{20.0} & \textbf{46.7}
& \textbf{48.1} & \textbf{80.7} \\

\bottomrule
\end{tabular}
\end{adjustbox}

\vspace{-3mm}
\end{table*}

\subsection{Additional Model Scale and Domains}

\label{sec:additional_models}

To answer \textbf{RQ4}, we extend PAPO to additional models,
scales, and different reward semantics and domains, including multimodal
safety alignment.

\textbf{Additional Model Scale} To examine whether PAPO generalizes beyond the primary model setting, we
further evaluate it on additional smaller-scale backbone. Table~\ref{tab:additional_math_model} further demonstrates the effectiveness of PAPO on DeepSeek-R1-Distill-1.5B~\citep{guo2025deepseek}. Compared with PPO, PAPO improves average accuracy from 45.8 to 48.1 and length compliance from 75.8 to 80.7, with consistent gains across most benchmarks. These results indicate that PAPO remains effective across different model families and reasoning backbones while jointly improving the primary and auxiliary objectives. We also provide results of Qwen3-1.7B for agentic tool use in Section~\ref{app:additional_model_results}.

\begin{wrapfigure}{r}{0.45\linewidth}
    \centering
    \vspace{-6mm}
    \includegraphics[width=1.0\linewidth]{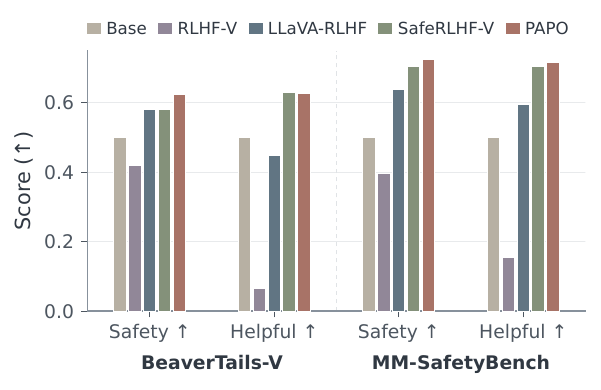}
    \vspace{-9mm}
    \caption{
    Multimodal safety alignment.
    }
    \label{fig:safety_results}
    \vspace{-3mm}
\end{wrapfigure}
\textbf{Additional Domain}
We further evaluate PAPO in the multimodal safety-alignment setting of
SafeRLHF-V~\citep{ji2025saferlhfv}, treating \emph{safety} as the primary
objective and \emph{helpfulness} as the auxiliary objective. We evaluate
Qwen2-VL-7B (in Section~\ref{app:additional_model_results}) and LLaVA-1.6-7B~\citep{liu2024llavanext} on BeaverTails-V~\citep{ji2025saferlhfv} and MM-SafetyBench~\citep{liu2024mm} against Base, RLHF~\citep{ouyang2022training}
RLHF-V \citep{sun2024aligning}, LLaVA-RLHF \citep{sun2024aligning}, and SafeRLHF-V \citep{ji2025saferlhfv}. Figure~\ref{fig:safety_results} shows that
PAPO consistently improves safety while preserving strong helpfulness, and
achieves a favorable trade-off across both model families. This further
demonstrates that PAPO generalizes to different rewards and domains.

Besides, we provide additional model results in Appendix~\ref{app:additional_model_results}, additional training results for the tool use in Appendix~\ref{app:tool_dynamics}, and per-dataset evaluation dynamics during training in Appendix~\ref{app:per_dataset_dynamics}. We conduct multiple runs with random seeds and evaluate pass@k in Appendix~\ref{app:passk} to complement the deterministic greedy-decoding evaluation above. Other implementation details and examples are listed in Appendix~\ref{app:training_details} and Appendix~\ref{app:training_examples}.

\section{Conclusion}

We propose \textbf{Primary Advantage Policy Optimization (PAPO)}, a priority-aware advantage aggregation framework that preserves aligned auxiliary supervision while smoothly constraining conflicting signals. PAPO applies to both GRPO and PPO with token-level optimization with objective-specific critic heads. Across mathematical reasoning, agentic tool use, and multimodal safety alignment, PAPO consistently improves primary-objective performance while maintaining strong auxiliary behavior, with gains that generalize across models, scales, reward semantics, and optimization frameworks. These results suggest that explicitly modeling reward objective priority is an effective way for multi-reward policy optimization.

 \clearpage
\newpage

\section*{Reproducibility Statement}

\textbf{Code.}
We provide the implementation, training configurations, and evaluation scripts
needed to reproduce the main experiments in the supplementary material. The
code is available \href{https://anonymous.4open.science/r/papo-BC29/}{here}, which covers both the GRPO- and PPO-based implementations of PAPO, including
the dual-head critic used in the PPO setting.

\textbf{Training and evaluation.}
Appendix~\ref{app:training_details} documents the datasets, reward definitions, model backbones,
training objectives, hyperparameters, and evaluation protocols used throughout
the experiments. We additionally provide implementation details for both
mathematical reasoning and agentic tool-use settings, including the decoding
protocols used for greedy and Pass@$k$ evaluation.

\textbf{Additional experiments.}
Appendix~\ref{app:exp_results} reports additional ablations and analyses, including the effects
of the auxiliary-reward weight, critic design, training and evaluation
dynamics, and Pass@$k$ evaluation under stochastic decoding. Prompt templates
and task-specific evaluation details are also provided to facilitate
reproduction of the reported results.

\section*{Ethics Statement}

This work studies multi-objective reinforcement learning for large language
model post-training across mathematical reasoning, agentic tool use, and safety
alignment. Our experiments do not involve human subjects or private user data,
and tool-use experiments are conducted in controlled benchmark environments.
PAPO focuses on prioritizing a primary objective while incorporating auxiliary
objectives, including settings where safety is primary and helpfulness is
secondary. We encourage future applications to carefully define the objective
hierarchy and evaluate whether it reflects the intended safety and deployment
requirements.

\section*{AI Use Statement}

Generative AI tools were used to assist with grammar correction, language
polishing, and improving the clarity and concision of the manuscript, as well
as for discovery of potentially relevant literature. We
reviewed all AI-assisted edits and independently verified suggested references
against the original sources. The authors take full responsibility for the
final text, citations, claims, experimental results, and presentation of the
work.

\bibliography{iclr2027_conference}

@article{schulman2017proximal,
  title={Proximal Policy Optimization Algorithms},
  author={Schulman, John and Wolski, Filip and Dhariwal, Prafulla and Radford, Alec and Klimov, Oleg},
  journal={arXiv preprint arXiv:1707.06347},
  year={2017}
}

@article{zheng2025gspo,
  title={Group Sequence Policy Optimization},
  author={Zheng, Chujie and Liu, Shixuan and Li, Mingze and Chen, Xiong-Hui and Yu, Bowen and Gao, Chang and Dang, Kai and Liu, Yuqiong and Men, Rui and Yang, An and Zhou, Jingren and Lin, Junyang},
  journal={arXiv preprint arXiv:2507.18071},
  year={2025}
}

@article{dai2023saferlhf,
  title={Safe RLHF: Safe Reinforcement Learning from Human Feedback},
  author={Dai, Josef and Pan, Xuehai and Sun, Ruiyang and Ji, Jiaming and Xu, Xinbo and Liu, Mickel and Wang, Yizhou and Yang, Yaodong},
  journal={arXiv preprint arXiv:2310.12773},
  year={2023}
}

@inproceedings{ichihara2025awgrpo,
  title={Auto-Weighted Group Relative Preference Optimization for Multi-Objective Text Generation Tasks},
  author={Ichihara, Yuki and Jinnai, Yuu},
  booktitle={Proceedings of the 2025 Conference on Empirical Methods in Natural Language Processing: Industry Track},
  pages={1134--1147},
  publisher={Association for Computational Linguistics},
  year={2025},
  doi={10.18653/v1/2025.emnlp-industry.80}
}

@article{liu2026gdpo,
  title={GDPO: Group reward-Decoupled Normalization Policy Optimization for Multi-reward RL Optimization},
  author={Liu, Shih-Yang and Dong, Xin and Lu, Ximing and Diao, Shizhe and Belcak, Peter and Liu, Mingjie and Chen, Min-Hung and Yin, Hongxu and Wang, Yu-Chiang Frank and Cheng, Kwang-Ting and Choi, Yejin and Kautz, Jan and Molchanov, Pavlo},
  journal={arXiv preprint arXiv:2601.05242},
  year={2026}
}

@article{wang2026smopd,
  title={SMOPD: Multi-Reward Reinforcement Learning via Specialize-and-Merge Online Policy Distillation},
  author={Wang, Wen and Bao, Jiahua and Tu, Yongsiqi and Liu, Yihao and Zhou, Haotian and Ma, Haoxuan and Zhou, Mengyu and Fan, Wenkui and He, Junwei and Jiang, Xiaoxi and Jiang, Guanjun},
  journal={arXiv preprint arXiv:2608.03092},
  year={2026}
}

@article{sun2026when2tool,
  title={LLM Agents Already Know When to Call Tools -- Even Without Reasoning},
  author={Sun, Chung-En and Liu, Linbo and Yan, Ge and Wang, Zimo and Weng, Tsui-Wei},
  journal={arXiv preprint arXiv:2605.09252},
  year={2026}
}

@techreport{microsoft2026mai,
  title={Mai-thinking-1: Building a hill-climbing machine},
  author={Microsoft AI Team and others},
  year={2026},
  institution={Technical report, Microsoft AI, 2026. https://microsoft. ai/pdf/mai-thinking~…}
}

@inproceedings{liu2026learn,
  title={Learn to reason efficiently with adaptive length-based reward shaping},
  author={Liu, Wei and Zhou, Ruochen and Deng, Yiyun and Huang, Yuzhen and Liu, Junteng and Deng, Yuntian and Zhang, Yizhe and He, Junxian},
  booktitle={International Conference on Learning Representations},
  year={2026}
}

@article{liu2025dler,
  title={Dler: Doing length penalty right-incentivizing more intelligence per token via reinforcement learning},
  author={Liu, Shih-Yang and Dong, Xin and Lu, Ximing and Diao, Shizhe and Liu, Mingjie and Chen, Min-Hung and Yin, Hongxu and Wang, Yu-Chiang Frank and Cheng, Kwang-Ting and Choi, Yejin and others},
  journal={arXiv preprint arXiv:2510.15110},
  year={2025}
}

@inproceedings{ji2025saferlhfv,
  title={Safe RLHF-V: Safe Reinforcement Learning from Multi-modal Human Feedback},
  author={Ji, Jiaming and Chen, Xinyu and Pan, Rui and Zhu, Han and
          Li, Jiahao and Hong, Donghai and Chen, Boyuan and Zhou, Jiayi and
          Wang, Kaile and Dai, Juntao and Chan, Chi-Min and Han, Sirui and
          Guo, Yike and Yang, Yaodong},
  booktitle={Advances in Neural Information Processing Systems},
  year={2025}
}

@article{li2024numinamath,
  title={Numinamath: The largest public dataset in ai4maths with 860k pairs of competition math problems and solutions},
  author={Li, Jia and Beeching, Edward and Tunstall, Lewis and Lipkin, Ben and Soletskyi, Roman and Huang, Shengyi and Rasul, Kashif and Yu, Longhui and Jiang, Albert Q and Shen, Ziju and others},
  journal={Hugging Face repository},
  year={2024}
}

@article{he2024olympiadbench,
  title={Olympiadbench: A challenging benchmark for promoting agi with olympiad-level bilingual multimodal scientific problems},
  author={He, Chaoqun and Luo, Renjie and Bai, Yuzhuo and Hu, Shengding and Thai, Zhen Leng and Shen, Junhao and Hu, Jinyi and Han, Xu and Huang, Yujie and Zhang, Yuxiang and others},
  journal={arXiv preprint arXiv:2402.14008},
  year={2024}
}

@article{hendrycks2021measuring,
  title={Measuring mathematical problem solving with the math dataset},
  author={Hendrycks, Dan and Burns, Collin and Kadavath, Saurav and Arora, Akul and Basart, Steven and Tang, Eric and Song, Dawn and Steinhardt, Jacob},
  journal={arXiv preprint arXiv:2103.03874},
  year={2021}
}

@article{lewkowycz2022solving,
  title={Solving quantitative reasoning problems with language models},
  author={Lewkowycz, Aitor and Andreassen, Anders and Dohan, David and Dyer, Ethan and Michalewski, Henryk and Ramasesh, Vinay and Slone, Ambrose and Anil, Cem and Schlag, Imanol and Gutman-Solo, Theo and others},
  journal={Advances in neural information processing systems},
  year={2022}
}

@article{yu2026dapo,
  title={Dapo: An open-source llm reinforcement learning system at scale},
  author={Yu, Qiying and Zhang, Zheng and Zhu, Ruofei and Yuan, Yufeng and Zuo, Xiaochen and Yue, Yu and Dai, Weinan and Fan, Tiantian and Liu, Gaohong and Liu, Lingjun and others},
  journal={Advances in Neural Information Processing Systems},
  year={2026}
}

@article{guo2025deepseek,
  title={Deepseek-r1: Incentivizing reasoning capability in llms via reinforcement learning},
  author={Guo, Daya and Yang, Dejian and Zhang, Haowei and Song, Junxiao and Zhang, Ruoyu and Xu, Runxin and Zhu, Qihao and Ma, Shirong and Wang, Peiyi and Bi, Xiao and others},
  journal={arXiv preprint arXiv:2501.12948},
  year={2025}
}

@article{ouyang2022training,
  title={Training language models to follow instructions with human feedback},
  author={Ouyang, Long and Wu, Jeffrey and Jiang, Xu and Almeida, Diogo and Wainwright, Carroll and Mishkin, Pamela and Zhang, Chong and Agarwal, Sandhini and Slama, Katarina and Ray, Alex and others},
  journal={Advances in neural information processing systems},
  year={2022}
}

@inproceedings{sun2024aligning,
  title={Aligning large multimodal models with factually augmented rlhf},
  author={Sun, Zhiqing and Shen, Sheng and Cao, Shengcao and Liu, Haotian and Li, Chunyuan and Shen, Yikang and Gan, Chuang and Gui, Liangyan and Wang, Yu-Xiong and Yang, Yiming and others},
  booktitle={Findings of the association for computational linguistics: ACL 2024},
  year={2024}
}

@article{wang2024qwen2,
  title={Qwen2-vl: Enhancing vision-language model's perception of the world at any resolution},
  author={Wang, Peng and Bai, Shuai and Tan, Sinan and Wang, Shijie and Fan, Zhihao and Bai, Jinze and Chen, Keqin and Liu, Xuejing and Wang, Jialin and Ge, Wenbin and others},
  journal={arXiv preprint arXiv:2409.12191},
  year={2024}
}

@article{yang2025qwen3,
  title={Qwen3 technical report},
  author={Yang, An and Li, Anfeng and Yang, Baosong and Zhang, Beichen and Hui, Binyuan and Zheng, Bo and Yu, Bowen and Gao, Chang and Huang, Chengen and Lv, Chenxu and others},
  journal={arXiv preprint arXiv:2505.09388},
  year={2025}
}

@misc{liu2024llavanext,
    title={LLaVA-NeXT: Improved reasoning, OCR, and world knowledge},
    url={https://llava-vl.github.io/blog/2024-01-30-llava-next/},
    author={Liu, Haotian and Li, Chunyuan and Li, Yuheng and Li, Bo and Zhang, Yuanhan and Shen, Sheng and Lee, Yong Jae},
    month={January},
    year={2024}
}

@inproceedings{
aggarwal2025l,
title={L1: Controlling How Long A Reasoning Model Thinks With Reinforcement Learning},
author={Pranjal Aggarwal and Sean Welleck},
booktitle={Second Conference on Language Modeling},
year={2025}
}

@inproceedings{liu2024mm,
  title={Mm-safetybench: A benchmark for safety evaluation of multimodal large language models},
  author={Liu, Xin and Zhu, Yichen and Gu, Jindong and Lan, Yunshi and Yang, Chao and Qiao, Yu},
  booktitle={European Conference on Computer Vision},
  pages={386--403},
  year={2024},
  organization={Springer}
}

@article{jiang2026dvao,
  title={DVAO: Dynamic Variance-adaptive Advantage Optimization for Multi-reward Reinforcement Learning},
  author={Jiang, Guochao and Song, Jingyi and Quan, Guofeng and Hao, Chuzhan and Liu, Guohua and Zhang, Yuewei},
  journal={arXiv preprint arXiv:2605.25604},
  year={2026}
}

@article{yue2025vapo,
  title={Vapo: Efficient and reliable reinforcement learning for advanced reasoning tasks},
  author={Yue, Yu and Yuan, Yufeng and Yu, Qiying and Zuo, Xiaochen and Zhu, Ruofei and Xu, Wenyuan and Chen, Jiaze and Wang, Chengyi and Fan, TianTian and Du, Zhengyin and others},
  journal={arXiv preprint arXiv:2504.05118},
  year={2025}
}

@article{ma2026fipo,
  title={Fipo: Eliciting deep reasoning with future-kl influenced policy optimization},
  author={Ma, Chiyu and Yang, Shuo and Huang, Kexin and Lu, Jinda and Meng, Haoming and Wang, Shangshang and Ding, Bolin and Vosoughi, Soroush and Wang, Guoyin and Zhou, Jingren},
  journal={arXiv preprint arXiv:2603.19835},
  year={2026}
}

@inproceedings{diao2026addressing,
  title={Addressing overthinking in large vision-language models via gated perception-reasoning optimization},
  author={Diao, Xingjian and Liu, Zheyuan and Zhang, Chunhui and Wu, Weiyi and Kong, Keyi and Shi, Lin and Ding, Kaize and Vosoughi, Soroush and Gui, Jiang},
  booktitle={Findings of the Association for Computational Linguistics: ACL 2026},
  year={2026}
}

@article{diao2026errors,
  title={Of Errors and Echoes: Responsibility-Aware Credit Assignment for Tool-Use Agents},
  author={Diao, Xingjian and Yang, Tianyu and Huang, Wenjun and Yuan, Xiangchi and Zhang, Chunhui and Zhao, Xingchen and Wu, Weiyi and Xu, Haotian and Vosoughi, Soroush and Zhang, Xiangliang and others},
  journal={arXiv preprint arXiv:2610.16027},
  year={2026}
}

@article{yuan2025mitigating,
  title={Mitigating forgetting between supervised and reinforcement learning yields stronger reasoners},
  author={Yuan, Xiangchi and Chen, Xiang and Yu, Tong and Shi, Dachuan and Jin, Can and Lee, Wenke and Mitra, Saayan},
  journal={arXiv preprint arXiv:2510.04454},
  year={2025}
}

@inproceedings{yuan-etal-2026-behavior,
    title = "Behavior Knowledge Merge in Reinforced Agentic Models",
    author = "Yuan, Xiangchi  and
      Shi, Dachuan  and
      Zhang, Chunhui  and
      Liu, Zheyuan  and
      Yao, Shenglong  and
      Vosoughi, Soroush  and
      Lee, Wenke",
    booktitle = "Proceedings of the 64th Annual Meeting of the {A}ssociation for {C}omputational {L}inguistics",
    year = "2026"
}

@inproceedings{
shi2026swireasoning,
title={SwiReasoning: Switch-Thinking in Latent and Explicit for Pareto-Superior Reasoning {LLM}s},
author={Dachuan Shi and Abedelkadir Asi and Keying Li and Xiangchi Yuan and Leyan Pan and Wenke Lee and Wen Xiao},
booktitle={The Fourteenth International Conference on Learning Representations},
year={2026}
}

@article{shi2026copt,
  title={CopT: Contrastive On-Policy Thinking with Continuous Spaces for General and Agentic Reasoning},
  author={Shi, Dachuan and Zhu, Hanlin and Yuan, Xiangchi and Zhao, Wanjia and Xia, Kejing and Xiao, Wen and Lee, Wenke},
  journal={arXiv preprint arXiv:2605.20075},
  year={2026}
}
\bibliographystyle{iclr2027_conference}

\clearpage
\appendix

\startcontents[appendices]
\section*{CONTENTS OF APPENDIX}
\printcontents[appendices]{}{1}{\normalsize}

\clearpage
\section{Proof}
\label{app:normalization_priority}

For completeness, we restate
Proposition~\ref{prop:papo_priority_main} and provide its full proof.

\begin{propositionbox}
\noindent
\textbf{Proposition~\ref{prop:papo_priority_main}.}
Let an optimization unit denote a rollout in GRPO or a token position in PPO;
the associated normalization set is a rollout group for GRPO and a token-level
mini-batch for PPO.
For any unit $u$ with $A_{u,\mathrm{pri}}\neq 0$, PAPO preserves the primary
sign during aggregation:
$\operatorname{sign}(A_{u,\mathrm{agg}})
=
\operatorname{sign}(A_{u,\mathrm{pri}})$.
Moreover, for any two units $u$ and $v$ in the same normalization set,
centered normalization
$\hat{A}_k=(A_{k,\mathrm{agg}}-\mu)/(\sigma+\epsilon)$
preserves their primary-consistent ordering:
\[
A_{u,\mathrm{pri}}>0>A_{v,\mathrm{pri}}
\quad\Longrightarrow\quad
\hat{A}_u>\hat{A}_v.
\]
\end{propositionbox}

\begin{proof}

\noindent\textit{Sign preservation during PAPO aggregation.}
Consider any optimization unit $u$ with
$A_{u,\mathrm{pri}}\neq 0$.
When the primary and auxiliary advantages are aligned,
$A_{u,\mathrm{pri}}A_{u,\mathrm{aux}}\ge 0$, PAPO reduces to
\[
A_{u,\mathrm{agg}}
=
A_{u,\mathrm{pri}}
+
\alpha A_{u,\mathrm{aux}}.
\]
Since $\alpha\ge0$ and the two terms have the same sign,
\[
\operatorname{sign}(A_{u,\mathrm{agg}})
=
\operatorname{sign}(A_{u,\mathrm{pri}}).
\]

When the two advantages conflict,
$A_{u,\mathrm{pri}}A_{u,\mathrm{aux}}<0$, Eq.~\ref{eq:papo} can be rewritten as
\begin{equation}
A_{u,\mathrm{agg}}
=
A_{u,\mathrm{pri}}
\left[
1-
\tanh\left(
\frac{
\alpha |A_{u,\mathrm{aux}}|
}{
|A_{u,\mathrm{pri}}|+\epsilon
}
\right)
\right].
\label{eq:papo_sign_form}
\end{equation}
Because
\[
0
\le
\tanh\left(
\frac{
\alpha |A_{u,\mathrm{aux}}|
}{
|A_{u,\mathrm{pri}}|+\epsilon
}
\right)
<1,
\]
the multiplicative factor in Eq.~\ref{eq:papo_sign_form} satisfies
\[
0
<
1-
\tanh\left(
\frac{
\alpha |A_{u,\mathrm{aux}}|
}{
|A_{u,\mathrm{pri}}|+\epsilon
}
\right)
\le 1.
\]
Hence,
\[
\operatorname{sign}(A_{u,\mathrm{agg}})
=
\operatorname{sign}(A_{u,\mathrm{pri}}).
\]
Therefore, PAPO cannot reverse the sign induced by the primary objective
during aggregation.

\medskip
\noindent\textit{Primary-consistent ordering after centered normalization.}
Consider two optimization units $u$ and $v$ in the same normalization set such
that
\[
A_{u,\mathrm{pri}}>0>A_{v,\mathrm{pri}}.
\]
By the sign-preservation result above,
\[
A_{u,\mathrm{agg}}>0>A_{v,\mathrm{agg}},
\]
and therefore
\[
A_{u,\mathrm{agg}}>A_{v,\mathrm{agg}}.
\]

The post-fusion centered normalization is
\begin{equation}
\hat{A}_k
=
\frac{
A_{k,\mathrm{agg}}-\mu
}{
\sigma+\epsilon
}.
\label{eq:post_norm}
\end{equation}
Since $\sigma+\epsilon>0$, Eq.~\ref{eq:post_norm} is a strictly increasing
affine transformation. Therefore,
\[
A_{u,\mathrm{agg}}>A_{v,\mathrm{agg}}
\quad\Longrightarrow\quad
\hat{A}_u>\hat{A}_v.
\]
Combining the two results gives
\[
A_{u,\mathrm{pri}}>0>A_{v,\mathrm{pri}}
\quad\Longrightarrow\quad
\hat{A}_u>\hat{A}_v.
\]
Thus, the primary-consistent ordering induced by PAPO is preserved after the
centered normalization used in both GRPO and PPO.

\end{proof}

Normalization does not necessarily preserve the absolute sign of each
individual fused advantage. For example, a positive
$A_{u,\mathrm{agg}}$ may become negative after subtracting the normalization
mean. This does not constitute an auxiliary-induced priority inversion: the
sign change results from shifting the common advantage reference point rather
than from the auxiliary objective overturning the primary preference.
Therefore, PAPO provides two complementary guarantees: it preserves the
primary sign during advantage aggregation, while the resulting
\emph{primary-consistent ordering} remains intact after centered normalization.



\section{Additional Experimental Results}
\label{app:exp_results}

\subsection {Additional Models' Results}
\label{app:additional_model_results}

\paragraph{Additional Tool-Use Results.} Table~\ref{tab:additional_tool_model} shows the tool-use results of applying PAPO and baselines to the Qwen3-1.7B~\citep{yang2025qwen3} model. PAPO achieves the highest overall accuracy of 81.11, compared with 75.96 for GRPO and 76.85 for GDPO, with notable gains on the medium and hard subsets. Although PAPO uses slightly more tool calls on average, the higher accuracy suggests that it avoids over-optimizing tool efficiency at the cost of accuracy, or task success rate, with comparable tool call efficiency.

\begin{table}[h]
\centering
\caption{
\textbf{Tool-use results on Qwen3-1.7B.}
Higher is better for accuracy, while lower is better for tool calls.
Best results are in \textbf{bold} and second-best results are
\underline{underlined}.
}
\vspace{2mm}
\label{tab:additional_tool_model}

\small
\setlength{\tabcolsep}{4.0pt}
\renewcommand{\arraystretch}{1.08}

\begin{adjustbox}{max width=\columnwidth,center}
\begin{tabular}{lcc|cc|cc|cc}
\toprule
& \multicolumn{2}{c|}{\textbf{Easy}}
& \multicolumn{2}{c|}{\textbf{Medium}}
& \multicolumn{2}{c|}{\textbf{Hard}}
& \multicolumn{2}{c}{\textbf{Overall}} \\
\cmidrule(lr){2-3}
\cmidrule(lr){4-5}
\cmidrule(lr){6-7}
\cmidrule(lr){8-9}

\textbf{Method}
& Acc. $\uparrow$ & Calls $\downarrow$
& Acc. $\uparrow$ & Calls $\downarrow$
& Acc. $\uparrow$ & Calls $\downarrow$
& Acc. $\uparrow$ & Calls $\downarrow$ \\
\midrule

Base
& 60.11 & 1.33
& 60.56 & 1.43
& 55.89 & 1.49
& 58.85 & 1.42 \\

GRPO
& 85.33 & 0.18
& 76.33 & 0.30
& 66.22 & 0.65
& 75.96 & 0.38 \\

GDPO
& \textbf{88.86} & \textbf{0.13}
& \underline{76.00} & \textbf{0.21}
& \underline{66.00} & \textbf{0.54}
& \underline{76.85} & \textbf{0.29} \\

PAPO
& \underline{88.00} & \underline{0.17}
& \textbf{80.78} & \underline{0.29}
& \textbf{74.56} & \underline{0.87}
& \textbf{81.11} & \underline{0.44} \\

\bottomrule
\end{tabular}
\end{adjustbox}

\vspace{-2mm}
\end{table}


\paragraph{Additional Safety Alignment Results.} We continue to conduct additional safety alignment experiments on the additional multi-modality model Qwen2-VL-7B~\citep{wang2024qwen2} and illustrate the results in Figure~\ref{fig:safety_results_qwen}. Compared with other baselines, our PAPO achieves the best primary safety performance while having the best or very competitive auxiliary helpful score on both datasets.

\begin{figure}[h]
    \centering

    \includegraphics[width=0.5\linewidth]{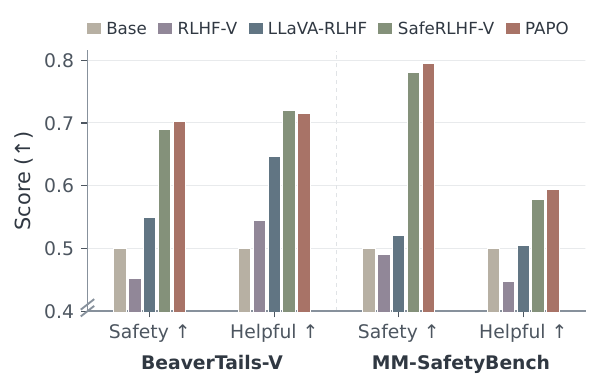}

    \caption{
    Multimodal safety alignment results for the Qwen2-VL-7B model.
    }
    \label{fig:safety_results_qwen}

\end{figure}


\subsection{Training and Evaluation Dynamics for Tool}
\label{app:tool_dynamics}
\begin{figure}[h]
    \centering
    \includegraphics[width=\linewidth]{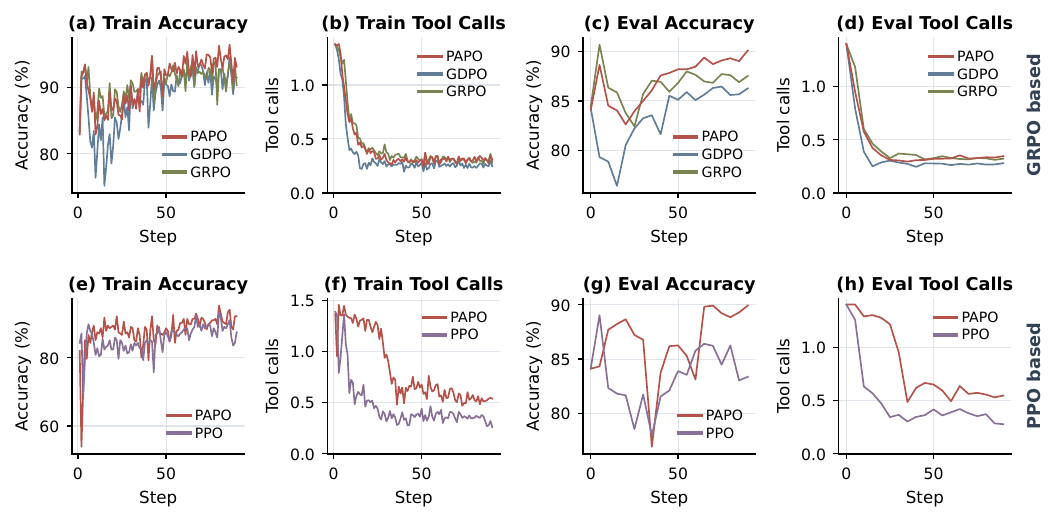}
    \caption{
    \textbf{Dynamics on agentic tool use.}
    Although baseline methods can reduce tool calls aggressively, this behavior
    can coincide with degraded accuracy.
    PAPO maintains stronger primary-objective performance while still reducing tool usage, yielding a better trade-off for GRPO-based
    (top) and PPO-based (bottom) methods.
    }
    \label{fig:tool_dynamics}
\end{figure}

Figure~\ref{fig:tool_dynamics} shows the optimization dynamics of task correctness and tool-use efficiency. For both GRPO- and PPO-based methods, the number of tool calls decreases rapidly during training, suggesting that the auxiliary efficiency objective is relatively easy to optimize; however, lower tool usage does not necessarily yield better performance in accuracy. PAPO maintains stronger and more stable evaluation accuracy while using a moderate number of tool calls, whereas GRPO, GDPO, and especially standard PPO can reduce tool usage more aggressively at the cost of weaker or more volatile correctness. 

\subsection{Per-Dataset Evaluation Dynamics}
\label{app:per_dataset_dynamics}

\begin{figure*}[h]
    \centering
    \includegraphics[width=0.98\textwidth]{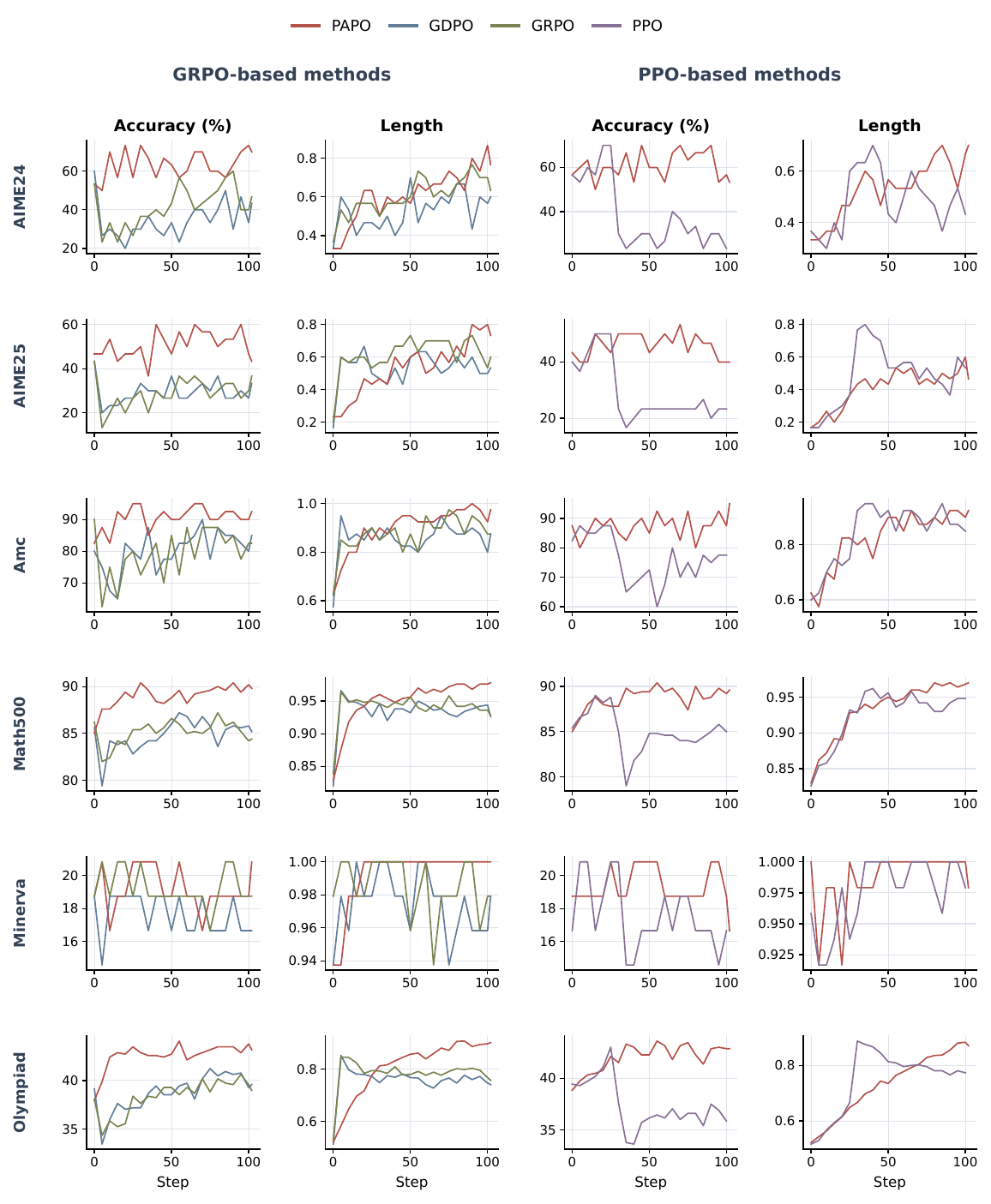}
    \caption{
    \textbf{Per-dataset evaluation dynamics on mathematical reasoning.}
    We report accuracy and length compliance throughout training for each
    benchmark under GRPO- and PPO-based optimization.
    }
    \label{fig:math_dataset_dynamics}

\end{figure*}

\begin{figure*}[h]
    \centering
    \includegraphics[width=0.98\textwidth]{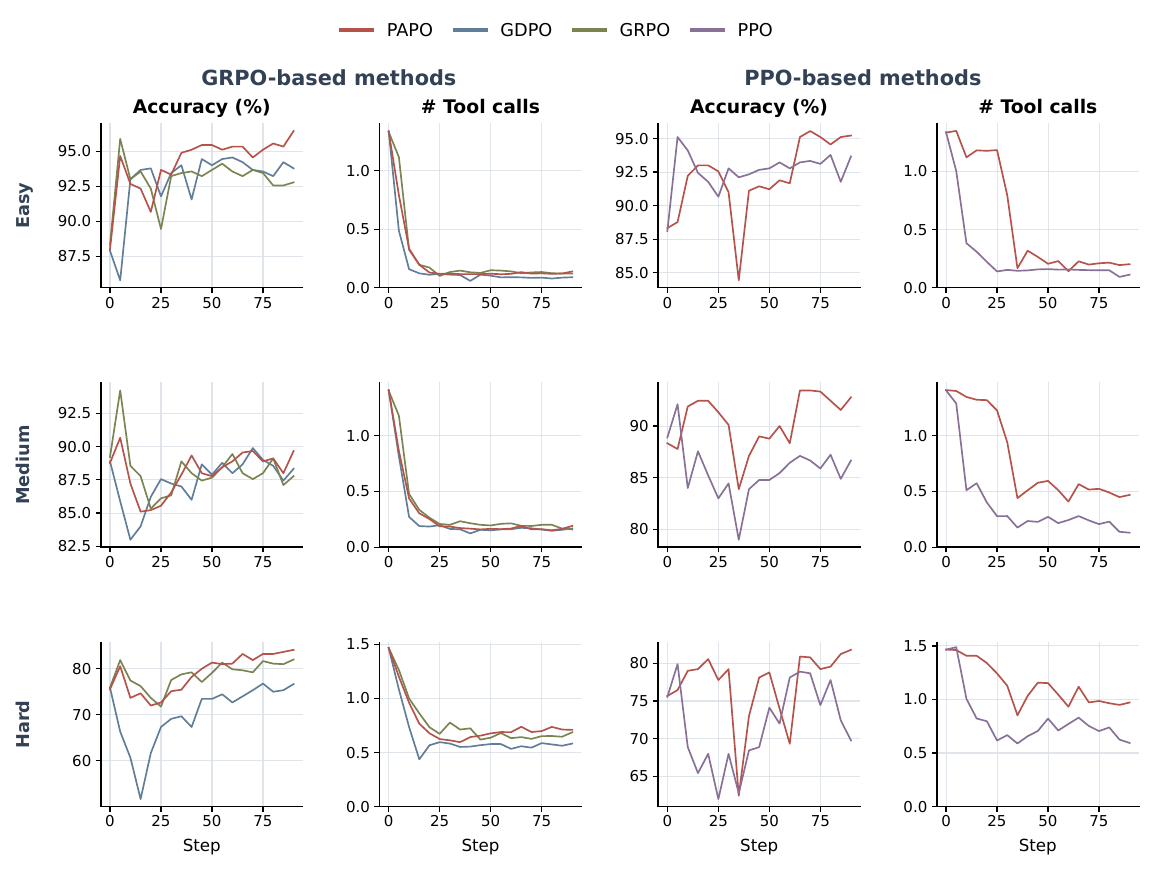}
    \caption{
    \textbf{Per-difficulty evaluation dynamics on agentic tool use.}
    We report accuracy and the number of tool calls for easy, medium, and hard
    instances under GRPO- and PPO-based optimization.
    }
    \label{fig:tool_dataset_dynamics}
\end{figure*}

To complement the aggregated training dynamics in Section~\ref{sec:dynamics} and Appendix~\ref{app:tool_dynamics}, we further
report evaluation dynamics during training separately for each benchmark (math) and difficulty level (tool using). These results provide a more fine-grained view of how PAPO balances the
primary and auxiliary objectives throughout optimization.

\paragraph{Mathematical Reasoning.}
Figure~\ref{fig:math_dataset_dynamics} presents the per-dataset evaluation
dynamics on AIME24, AIME25, AMC, MATH-500, Minerva, and Olympiad.
Across both optimization families, the baselines tend to improve the auxiliary
length objective relatively quickly, but this does not consistently bring stronger correctness. Under GRPO-based optimization, GRPO and GDPO often
reach comparable or higher length compliance early in training while remaining
substantially below PAPO in accuracy, particularly on AIME24, AIME25, AMC,
MATH-500, and Olympiad. PAPO instead maintains stronger accuracy while
progressively improving length compliance, often matching or surpassing the
baselines toward the end of training. The same pattern is even more pronounced
for PPO: PPO rapidly increases the length reward, but this is
accompanied by large accuracy drops on evaluation benchmarks, whereas PAPO has
a more gradual efficiency improvement and preserves very stable
correctness improvements. Minerva is less discriminative because length compliance quickly
approaches its ceiling and accuracy varies over a narrow discrete range.
Overall, the per-dataset trajectories show that GRPO- and
PPO-based baselines aggressively optimizing the
auxiliary length objective, interfering with correctness, while PAPO better preserves the primary accuracy objective as efficiency improves.

\paragraph{Agentic Tool Use.}
Figure~\ref{fig:tool_dataset_dynamics} reports evaluation dynamics separately
for easy, medium, and hard tool-use subsets. All methods rapidly reduce tool
calls, indicating that the auxiliary efficiency objective is comparatively easy to optimize. However, the difference in accuracy becomes increasingly apparent
as task difficulty grows. PAPO maintains strong accuracy across all three
difficulty levels and is particularly robust on medium and hard subsets,
where more aggressive reductions in tool usage can coincide with unstable or
lower accuracy. This effect is especially visible in the PPO-based setting,
where PPO often reaches fewer tool calls earlier but exhibits considerably more volatile correctness. PAPO instead retains additional tool usage when needed, achieving our intended priority: tool efficiency should refine the solution process without overriding successful task completion.

\subsection{Additional Pass@k Evaluation for Math}
\label{app:passk}

Table~\ref{tab:math_passk} provides a complementary Pass@$k$ evaluation on the
mathematical reasoning benchmarks. For each problem, we sample $n$ independent
responses and define
\[
\mathrm{Pass@}k
=
\frac{1}{N}
\sum_{i=1}^{N}
\mathbb{I}\!\left[
\max_{j\in\{1,\ldots,k\}} r_{i,j}=1
\right],
\]
where $r_{i,j}\in\{0,1\}$ indicates whether the $j$-th sampled response for
problem $i$ is correct. Thus, Pass@$k$ measures the fraction of problems for
which at least one of the first $k$ sampled responses is correct.
Unlike the main results in Table~\ref{tab:main_results_transposed}, which use
greedy decoding without sampling to gain stable evaluation results, Pass@$k$ requires multiple candidate
responses for each prompt and therefore uses stochastic decoding. Consequently,
the absolute values in this table are not directly comparable to those in the
main evaluation. We use greedy decoding for the main results because it
evaluates the learned policy through a deterministic single-response output,
avoids additional variance from sampling, and provides a cleaner comparison
across optimization methods. Pass@$k$ instead serves as a complementary
evaluation of the model's ability to produce a correct solution across multiple
sampled generations.

Importantly, PAPO remains consistently competitive under this different
decoding protocol. On the per-dataset average, the two PAPO variants occupy the
top two positions for every Pass@$k$ metric: PAPO-GRPO achieves the highest
Pass@1 and Pass@2, while PAPO-PPO achieves the highest Pass@3--5. Specifically,
PAPO-PPO improves average Pass@1 from 50.08 to 59.08 over PPO and reaches
67.90 at Pass@5, while PAPO-GRPO improves Pass@1 from 52.49 to 59.97 over GRPO
and reaches 67.63 at Pass@5. Meanwhile, both variants maintain strong
length-budget compliance, reaching 83.42\% for PAPO-PPO and 89.86\% for
PAPO-GRPO. These results indicate that PAPO's improvements are not specific to
deterministic greedy decoding: its advantage persists under stochastic
sampling, improving the probability of obtaining correct solutions while
retaining favorable length efficiency.

\begin{table*}[!h]
\centering
\caption{
\textbf{Pass@$k$ and length-compliance results on mathematical reasoning benchmarks.}
We report Pass@$k$ (\%) for $k\in\{1,2,3,4,5\}$ and the percentage of responses
within the 8K-token length budget (Len.~$\leq$8K).
The best result for each metric is shown in \textbf{bold}, and the second-best
result is \underline{underlined}.
}
\vspace{2mm}
\label{tab:math_passk}

\scriptsize
\setlength{\tabcolsep}{4.2pt}
\renewcommand{\arraystretch}{1.05}

\resizebox{\textwidth}{!}{%
\begin{tabular}{llcccccc}
\toprule
\textbf{Dataset} & \textbf{Method}
& \textbf{Pass@1}
& \textbf{Pass@2}
& \textbf{Pass@3}
& \textbf{Pass@4}
& \textbf{Pass@5}
& \textbf{Len.~$\leq$8K} \\
\midrule

\multirow{6}{*}{\textbf{AIME24}}
& Base
& \underline{50.67} & 60.00 & 63.00 & 63.33 & 63.33 & 30.00 \\
& PPO
& 41.33 & 48.67 & 53.67 & 57.33 & 60.00 & 84.67 \\
& GRPO
& 44.67 & 55.00 & 59.67 & 63.33 & 66.67 & \underline{90.67} \\
& GDPO
& 46.00 & 53.00 & 57.33 & 60.67 & 63.33 & \textbf{96.00} \\
\cmidrule(lr){2-8}
\rowcolor{gray!10}
\cellcolor{white}
& PAPO-PPO \textit{(Ours)}
& \textbf{63.33}
& \textbf{73.33}
& \underline{76.00}
& \underline{76.67}
& \underline{76.67}
& 68.00 \\
\rowcolor{gray!10}
\cellcolor{white}
& PAPO-GRPO \textit{(Ours)}
& \textbf{63.33}
& \underline{73.00}
& \textbf{77.00}
& \textbf{79.33}
& \textbf{80.00}
& 77.33 \\
\midrule

\multirow{6}{*}{\textbf{AIME25}}
& Base
& 42.00 & 49.67 & 52.33 & 53.33 & 53.33 & 17.33 \\
& PPO
& 32.67 & 42.00 & 48.00 & 52.67 & \underline{56.67} & \underline{92.67} \\
& GRPO
& 35.33 & 44.67 & 50.33 & 54.00 & \underline{56.67} & \textbf{94.00} \\
& GDPO
& 30.67 & 39.00 & 44.00 & 47.33 & 50.00 & \textbf{94.00} \\
\cmidrule(lr){2-8}
\rowcolor{gray!10}
\cellcolor{white}
& PAPO-PPO \textit{(Ours)}
& \underline{48.00}
& \underline{57.33}
& \underline{63.33}
& \underline{67.33}
& \textbf{70.00}
& 56.67 \\
\rowcolor{gray!10}
\cellcolor{white}
& PAPO-GRPO \textit{(Ours)}
& \textbf{52.67}
& \textbf{60.67}
& \textbf{65.33}
& \textbf{68.00}
& \textbf{70.00}
& 76.67 \\
\midrule

\multirow{6}{*}{\textbf{AMC}}
& Base
& 85.00 & 90.25 & 92.75 & 94.00 & \underline{95.00} & 62.00 \\
& PPO
& 83.00 & 91.75 & 94.25 & \underline{95.00} & \underline{95.00} & 96.50 \\
& GRPO
& 85.00 & 91.50 & 93.50 & 94.50 & \underline{95.00} & \textbf{99.50} \\
& GDPO
& \underline{87.50} & 90.50 & 91.50 & 92.00 & 92.50 & \underline{98.50} \\
\cmidrule(lr){2-8}
\rowcolor{gray!10}
\cellcolor{white}
& PAPO-PPO \textit{(Ours)}
& \textbf{91.50}
& \textbf{95.00}
& \textbf{96.25}
& \textbf{97.00}
& \textbf{97.50}
& 91.00 \\
\rowcolor{gray!10}
\cellcolor{white}
& PAPO-GRPO \textit{(Ours)}
& \textbf{91.50}
& \underline{94.50}
& \underline{95.00}
& \underline{95.00}
& \underline{95.00}
& 95.50 \\
\midrule

\multirow{6}{*}{\textbf{MATH-500}}
& Base
& 87.32 & 89.70 & 90.72 & 91.36 & 91.80 & 83.80 \\
& PPO
& 87.32 & 90.92 & 92.00 & 92.48 & 92.80 & 99.00 \\
& GRPO
& 88.16
& \textbf{91.40}
& \textbf{92.46}
& \textbf{93.00}
& \textbf{93.40}
& \underline{99.28} \\
& GDPO
& 88.20
& 91.12
& \underline{92.16}
& \underline{92.68}
& \underline{93.00}
& \textbf{99.60} \\
\cmidrule(lr){2-8}
\rowcolor{gray!10}
\cellcolor{white}
& PAPO-PPO \textit{(Ours)}
& \textbf{89.68}
& \underline{91.34}
& 91.84
& 92.16
& 92.40
& 97.16 \\
\rowcolor{gray!10}
\cellcolor{white}
& PAPO-GRPO \textit{(Ours)}
& \underline{89.48}
& 90.86
& 91.50
& 91.92
& 92.20
& 97.56 \\
\midrule

\multirow{6}{*}{\textbf{Minerva}}
& Base
& \underline{18.75}
& \underline{20.00}
& \underline{20.62}
& \underline{20.83}
& \underline{20.83}
& \underline{94.17} \\
& PPO
& 16.25
& 17.92
& 18.54
& 18.75
& 18.75
& \textbf{100.00} \\
& GRPO
& \textbf{19.17}
& 19.58
& 20.00
& 20.42
& \underline{20.83}
& \textbf{100.00} \\
& GDPO
& 18.33
& \textbf{20.21}
& \underline{20.62}
& \underline{20.83}
& \underline{20.83}
& \textbf{100.00} \\
\cmidrule(lr){2-8}
\rowcolor{gray!10}
\cellcolor{white}
& PAPO-PPO \textit{(Ours)}
& \underline{18.75}
& \textbf{20.21}
& \textbf{21.25}
& \textbf{22.08}
& \textbf{22.92}
& \textbf{100.00} \\
\rowcolor{gray!10}
\cellcolor{white}
& PAPO-GRPO \textit{(Ours)}
& \underline{18.75}
& 19.58
& 20.00
& 20.42
& \underline{20.83}
& \textbf{100.00} \\
\midrule

\multirow{6}{*}{\textbf{Olympiad}}
& Base
& 39.82 & 41.89 & 42.95 & 43.66 & 44.20 & 52.38 \\
& PPO
& 39.88 & 43.79 & 45.39 & 46.34 & 47.02 & \underline{97.71} \\
& GRPO
& 42.59 & 45.39 & 46.50 & 47.14 & 47.62 & \underline{97.71} \\
& GDPO
& 43.01 & 45.49 & 46.50 & 47.14 & 47.62 & \textbf{98.04} \\
\cmidrule(lr){2-8}
\rowcolor{gray!10}
\cellcolor{white}
& PAPO-PPO \textit{(Ours)}
& \underline{43.21}
& \underline{45.86}
& \textbf{46.93}
& \textbf{47.53}
& \textbf{47.92}
& 87.68 \\
\rowcolor{gray!10}
\cellcolor{white}
& PAPO-GRPO \textit{(Ours)}
& \textbf{44.08}
& \textbf{45.91}
& \underline{46.74}
& \underline{47.29}
& \underline{47.77}
& 92.08 \\
\midrule


\multirow{6}{*}{\textbf{Avg.}}
& Base
& 53.93 & 58.58 & 60.40 & 61.09 & 61.42 & 56.61 \\
& PPO
& 50.08 & 55.84 & 58.64 & 60.43 & 61.71 & 95.09 \\
& GRPO
& 52.49 & 57.92 & 60.41 & 62.07 & 63.36 & \underline{96.86} \\
& GDPO
& 52.28 & 56.55 & 58.69 & 60.11 & 61.21 & \textbf{97.69} \\
\cmidrule(lr){2-8}
\rowcolor{gray!10}
\cellcolor{white}
& PAPO-PPO \textit{(Ours)}
& \underline{59.08}
& \underline{63.85}
& \textbf{65.93}
& \textbf{67.13}
& \textbf{67.90}
& 83.42 \\
\rowcolor{gray!10}
\cellcolor{white}
& PAPO-GRPO \textit{(Ours)}
& \textbf{59.97}
& \textbf{64.09}
& \underline{65.93}
& \underline{66.99}
& \underline{67.63}
& 89.86 \\

\bottomrule
\end{tabular}%
}
\end{table*}

\section{Training and Implementation Details}
\label{app:training_details}

We provide the training configurations used in our mathematical reasoning and
agentic tool-use experiments. Within each optimization family, all methods use
the same data, rollout configuration, training budget, and optimization
hyperparameters unless otherwise specified. PAPO only changes how the
objective-specific advantages are aggregated and, for PPO, the corresponding
critic parameterization. All experiments are conducted on a single node with
8 GPUs.

\subsection{Mathematical Reasoning}

\paragraph{GRPO-based methods.}
For the GRPO-based mathematical reasoning experiments
(\textsc{GRPO}, \textsc{GDPO}, and \textsc{PAPO-GRPO}), we train on DAPO-17K
and evaluate on the mathematical reasoning benchmarks described in the main
paper. The maximum prompt and response lengths are 1,024 and 12,000 tokens,
respectively. The length-reward threshold is fixed to 8,000 response tokens.
For each prompt, we sample $G=4$ responses. We use a global training batch size
of 512, a policy mini-batch size of 64, and a micro-batch size of 8 per GPU.
The actor learning rate is $1\times10^{-5}$, and all models are trained for
3 epochs. We use asymmetric PPO clipping with lower and upper clipping
coefficients of $0.20$ and $0.28$, respectively. The policy objective includes
an MSE-style KL regularization term with coefficient $5\times10^{-4}$, while
KL is not incorporated into the environmental reward. For standard GRPO, the scalar training reward is the sum of correctness and
length rewards with equal weights;
For GDPO, we follow the implementations in ~\cite{liu2026gdpo} to make correctness and length retained as separate reward dimensions,
independently normalized within each rollout group. For PAPO-GRPO, correctness is
treated as the primary objective and length as the auxiliary objective. Unless
otherwise specified, we use an auxiliary coefficient $\alpha=1.0$. Rollout
groups without effective reward variation are filtered during GRPO-based
training. Other detailed hyperparameters are listed in Table~\ref{tab:math_training_hparams}.

\paragraph{PPO-based methods.}
The PPO and PAPO-PPO experiments use the same training data,
sequence-length limits, batch sizes, learning rate, clipping range, KL
regularization, and 3-epoch training budget as the GRPO-based experiments.
The main difference is that PPO uses one rollout per prompt ($G=1$) and
introduces a learned critic. The critic learning rate is $1\times10^{-5}$
with no critic warm-up. PPO uses a single value head trained on the scalar return associated
with$R = R_{\mathrm{acc}} + R_{\mathrm{len}}.$
PAPO-PPO instead uses the dual-head critic introduced in Section~\ref{sec:method}: one
value head estimates the correctness return, and the other estimates the length
return. The two critic losses are equally weighted, which only determines the relative contribution of the
two value-function losses and is independent of the auxiliary-advantage
coefficient used by the actor. We use $\alpha=1.0$ for PAPO-PPO unless
otherwise specified. Other detailed hyperparameters are listed in Table~\ref{tab:math_training_hparams}.

\begin{table}[h]
\centering
\caption{Training hyperparameters for mathematical reasoning experiments.
Settings shared across methods are held fixed within each optimization family.}
\vspace{1mm}
\label{tab:math_training_hparams}
\small
\begin{tabular}{lcc}
\toprule
\textbf{Hyperparameter} & \textbf{GRPO-based} & \textbf{PPO-based} \\
\midrule
Training batch size & 512 & 512 \\
Policy mini-batch size & 64 & 64 \\
Micro-batch / GPU & 8 & 8 \\
Rollouts / prompt & 4 & 1 \\
Maximum prompt length & 1,024 & 1,024 \\
Maximum response length & 12,000 & 12,000 \\
Length-reward threshold & 8,000 & 8,000 \\
Actor learning rate & $1\times10^{-5}$ & $1\times10^{-5}$ \\
Critic learning rate & -- & $1\times10^{-5}$ \\
Training epochs & 3 & 3 \\
PPO clipping range & $[0.20, 0.28]$ & $[0.20, 0.28]$ \\
Actor KL coefficient & $5\times10^{-4}$ & $5\times10^{-4}$ \\
KL type & MSE & MSE \\
PAPO auxiliary coefficient $\alpha$ & 1.0 & 1.0 \\
\bottomrule
\end{tabular}
\end{table}

\subsection{Agentic Tool Use}

\paragraph{GRPO-based methods.}
For the When2Tool experiments, the training and evaluation splits contain
1,080 and 2,700 instances, respectively, combining both single-hop and
multi-hop tasks. The maximum prompt length is 4,096 tokens and the maximum
generated response length is 2,048 tokens.
We sample $G=8$ trajectories per prompt as group rollouts. We use
a global training batch size of 256, a policy mini-batch size of 32, and a
micro-batch size of 8 per GPU. The actor learning rate is
$1\times10^{-5}$, and models are trained for 20 epochs. We do not apply either
reward-level KL penalties or actor-level KL regularization in the GRPO-based
tool-use experiments.
During training, the agent is allowed at most 5 assistant turns and 4 user/environment turns,
with a maximum tool-response length of 2,048 tokens. We allow at most one
parallel tool call and disable explicit thinking mode during rollout.
As described in Section~\ref{sec:exp_tool}, answer correctness is the primary reward and
tool-use efficiency is the auxiliary reward. Standard GRPO optimizes the
scalar reward
$R = R_{\mathrm{pri}} + R_{\mathrm{aux}}.$
GDPO retains the two reward dimensions separately and combines their normalized
advantages with equal weights. For PAPO-GRPO, we use correctness as the
primary objective and tool efficiency as the auxiliary objective with
$\alpha=1$. The tool-budget offset in the efficiency reward is fixed to
$\delta=1$. Other detailed hyperparameters are listed in Table~\ref{tab:tool_training_hparams}.

\paragraph{PPO-based methods.}
PPO-based tool-use experiments use the same dataset, batch configuration,
sequence lengths, multi-turn interaction limits, and 20-epoch training budget
as the GRPO-based experiments. PPO samples one trajectory per prompt and uses
an actor and learned critic with learning rates of $1\times10^{-5}$ for both.
We use asymmetric PPO clipping with lower and upper coefficients of $0.20$ and
$0.28$, respectively.
Standard PPO optimizes the scalar combination $R = R_{\mathrm{pri}} + R_{\mathrm{aux}},$
whereas PAPO-PPO maintains the two reward dimensions separately and applies
PAPO to their corresponding advantages. PAPO-PPO uses the dual-head critic,
with one value head for answer correctness and one for tool efficiency. We set
the auxiliary coefficient to $\alpha=0.3$ for PAPO-PPO. As in the GRPO-based
setting, the tool-budget offset is $\delta=1$. Other detailed hyperparameters are listed in Table~\ref{tab:tool_training_hparams}.

\begin{table}[h]
\centering
\caption{Training hyperparameters for the When2Tool experiments.}
\vspace{1mm}
\label{tab:tool_training_hparams}
\small
\begin{tabular}{lcc}
\toprule
\textbf{Hyperparameter} & \textbf{GRPO-based} & \textbf{PPO-based} \\
\midrule
Training batch size & 256 & 256 \\
Policy mini-batch size & 32 & 32 \\
Micro-batch / GPU & 8 & 8 \\
Rollouts / prompt & 8 & 1 \\
Maximum prompt length & 4,096 & 4,096 \\
Maximum response length & 2,048 & 2,048 \\
Maximum assistant turns & 5 & 5 \\
Maximum environment turns & 4 & 4 \\
Maximum tool-response length & 2,048 & 2,048 \\
Maximum parallel tool calls & 1 & 1 \\
Actor learning rate & $1\times10^{-5}$ & $1\times10^{-5}$ \\
Critic learning rate & -- & $1\times10^{-5}$ \\
Training epochs & 20 & 20 \\
PPO clipping range & -- & $[0.20,0.28]$ \\
KL regularization & None & None \\
Tool-budget offset $\delta$ & 1.0 & 1.0 \\
PAPO auxiliary coefficient $\alpha$ & 1 & 0.3 \\
\bottomrule
\end{tabular}
\end{table}

\subsection{Safety Alignment}
\label{app:safety_alignment_details}

We evaluate PAPO in multimodal safety alignment following the
experimental setup of Safe RLHF-V~\citep{ji2025saferlhfv}. We use
BeaverTails-V as the preference dataset and retain its decoupled preference
modeling setup, where a reward model (RM-V) provides the helpfulness signal
and a cost model (CM-V) provides the safety signal. Following the original
implementation, the preference models are initialized from LLaVA-1.5-7B, and
we use PPO with GAE for policy optimization. We also retain the reference-policy
KL regularization and the PTX objective based on
Align-Anything-TI2T-Instruction-100K, so that the comparison isolates the
effect of multi-objective advantage aggregation.
To instantiate PAPO, we treat \emph{safety} as the primary objective and
\emph{helpfulness} as the auxiliary objective. Since the original CM-V assigns
a cost to unsafe behavior, we convert it into a maximization-oriented safety
signal by negating the cost. Specifically, for each token $t$, we independently
compute the safety and helpfulness advantages using GAE,
denoted by $A_{t,\mathrm{pri}}$ and $A_{t,\mathrm{aux}}$, respectively.
These two advantages are estimated with separate value heads and normalized
independently before aggregation. We then apply PAPO token-wise by Eq.~\ref{eq:papo_ppo}
The fused advantages are subsequently normalized following the same PPO
pipeline used in our main experiments and are used in the clipped actor
objective.
All remaining RL regularization components are kept unchanged to
enable a controlled comparison. Unless otherwise specified, we follow the original Safe RLHF-V training setup
and official implementation for all remaining optimization hyperparameters,
including the PPO configuration, KL regularization, PTX objective, batch
configuration, and training schedule. Thus, the primary modification is the
replacement of the original Lagrangian reward--cost combination with PAPO's
priority-aware advantage aggregation.

\subsection{Dataset Details and Statistics}
\label{app:dataset_details}

We provide additional details on the training and evaluation datasets used
throughout our experiments. Beyond reporting dataset sizes, we summarize the
task composition, difficulty structure, and evaluation scope of each domain.
The three settings differ substantially in both modality and objective
semantics: mathematical reasoning evaluates correctness under a generation
budget, agentic tool use evaluates task success under interaction cost, and
multimodal safety alignment evaluates safety jointly with helpfulness.


\paragraph{Competition-level Mathematical Reasoning.}
We use DAPO-Math-17K~\citep{yu2026dapo} as the training set for the
mathematical reasoning experiments. DAPO-Math-17K contains approximately
17K competition-style mathematical problems paired with verifiable reference
answers and was originally constructed for reinforcement-learning post-training.
The problems span common competition-mathematics domains such as algebra,
geometry, number theory, and combinatorics. Its deterministic answer format
makes it suitable for outcome-based RL, as correctness can be evaluated without
an external judge model.
We evaluate on six held-out benchmarks that cover substantially different
difficulty regimes and sources. AIME 2024 and AIME 2025 contain 30 problems
each from the corresponding American Invitational Mathematics Examinations and
primarily test difficult high-school competition mathematics. AMC23 contains
40 problems from the 2023 American Mathematics Competitions and represents a
relatively easier competition-level regime. MATH-500 is a curated 500-problem
subset of MATH~\citep{hendrycks2021measuring}, spanning algebra, geometry,
number theory, counting and probability, intermediate algebra, and
precalculus. Minerva Math~\citep{lewkowycz2022solving} contains 272
quantitative problems drawn primarily from university-level STEM material,
including subjects such as differential equations, physics, astronomy, and
chemistry. Finally, we use the \texttt{OE\_TO\_maths\_en\_COMP} subset of
OlympiadBench~\citep{he2024olympiadbench}, containing 674 English,
text-only, open-ended competition problems drawn from Olympiad-level
mathematics. Table~\ref{tab:math_dataset_details} summarizes the resulting training and
evaluation suite. All six evaluation sets are used only for evaluation and do
not contribute to training.

\begin{table*}[t]
\centering
\caption{
\textbf{Datasets used for mathematical reasoning.}
DAPO-Math-17K is used for policy optimization, while the remaining datasets
are held out for evaluation. The evaluation suite spans standard competition
mathematics, Olympiad-level problems, and university-level quantitative
reasoning.
}
\vspace{1mm}
\label{tab:math_dataset_details}
\small
\setlength{\tabcolsep}{5pt}
\renewcommand{\arraystretch}{1.10}
\begin{tabular}{lccp{7.3cm}}
\toprule
\textbf{Dataset}
& \textbf{\# Examples}
& \textbf{Role}
& \textbf{Scope} \\
\midrule

DAPO-Math-17K
& $\sim$17K
& Train
& Competition-style mathematics with verifiable reference answers; covers
  algebra, geometry, number theory, combinatorics, and related topics. \\

\midrule

AMC23
& 40
& Test
& 2023 AMC competition problems; comparatively short high-school
  competition mathematics across algebra, geometry, number theory, and
  combinatorics. \\

AIME 2024
& 30
& Test
& AIME I and II 2024; difficult competition problems with integer final
  answers. \\

AIME 2025
& 30
& Test
& AIME I and II 2025; difficult competition problems with integer final
  answers. \\

MATH-500
& 500
& Test
& Curated subset of MATH covering algebra, geometry, number theory,
  counting/probability, intermediate algebra, and precalculus. \\

Minerva Math
& 272
& Test
& University-level quantitative reasoning from STEM subjects, including
  mathematics, physics, astronomy, and chemistry. \\

OlympiadBench
& 674
& Test
& English text-only open-ended competition mathematics
  (\texttt{OE\_TO\_maths\_en\_COMP}), emphasizing Olympiad-level reasoning. \\

\bottomrule
\end{tabular}
\vspace{-2mm}
\end{table*}


\paragraph{Agentic Tool Use.}
We use \textsc{When2Tool}~\citep{sun2026when2tool} to study multi-reward
optimization when the auxiliary objective directly controls interaction cost.
Unlike conventional tool-use benchmarks in which every example is constructed
to require a tool, \textsc{When2Tool} explicitly varies whether external tool
use is necessary. It contains 18 environments: 15 single-hop environments and
three multi-hop environments. The benchmark organizes tool necessity around
three forms of model self-assessment:
\emph{computational scale}, which asks whether the model can reliably perform
the required computation internally;
\emph{knowledge boundaries}, which asks whether the required information is
available in the model's parametric knowledge; and
\emph{execution reliability}, which asks whether the model can accurately
execute a procedural operation without external assistance.
Each category contains five single-hop environments. Computational-scale tasks
include arithmetic, statistics, combinatorial counting, matrix operations, and
primality testing. Knowledge-boundary tasks cover retrieval, historical facts,
game rules, hashing, and decoding. Execution-reliability tasks require
manipulating lists, performing date calculations, executing code, resolving
schedules, or matching regular expressions. In addition, one three-hop
environment is constructed for each category by chaining calculator, retriever,
or code-execution operations.
Importantly, the easy, medium, and hard splits are not merely generic notions
of problem difficulty. They are designed to move the problem across a
\emph{tool-necessity boundary}. Easy examples are often directly solvable by
the language model, medium examples lie near the boundary where tool use becomes
beneficial, and hard examples deliberately increase computational scale,
introduce unavailable information, or require sufficiently complex execution
that tool use becomes substantially more important.

Table~\ref{tab:when2tool_environments} provides the full environment taxonomy
and illustrates how difficulty changes within each environment.
\begin{table*}[t]
\centering
\caption{
\textbf{Environment composition of \textsc{When2Tool}.}
Each single-hop environment contains easy, medium, and hard instances.
The last column summarizes how each environment moves from tasks that are
often solvable internally to tasks for which tool assistance becomes
increasingly useful or necessary. The three multi-hop environments contain
three dependent operations.
}
\vspace{1mm}
\label{tab:when2tool_environments}

\scriptsize
\setlength{\tabcolsep}{3.5pt}
\renewcommand{\arraystretch}{1.08}

\begin{tabularx}{\textwidth}{
@{}
p{2.0cm}
p{2.0cm}
p{3.0cm}
Y
@{}
}
\toprule
\textbf{Category}
& \textbf{Environment}
& \textbf{Capability}
& \textbf{Easy $\rightarrow$ Hard progression} \\
\midrule

\multirow{5}{*}{\makecell[l]{Computational\\Scale}}
& Calculator
& Arithmetic evaluation
& Small integer arithmetic $\rightarrow$ large multi-operation or
  trillion-scale expressions. \\

& Statistics
& Statistical computation
& Mean/median of short lists $\rightarrow$ variance, standard deviation,
  or correlation over longer lists. \\

& Counting
& Combinatorial counting
& Small combinations such as $\binom{6}{4}$ $\rightarrow$
  large combinatorial quantities such as $\binom{80}{40}$. \\

& Matrix
& Matrix computation
& Small $2\times2$ operations $\rightarrow$ larger-matrix determinants
  and other matrix operations. \\

& Prime
& Primality testing
& Small familiar integers $\rightarrow$ large integers whose primality
  is difficult to determine reliably without computation. \\

\midrule

\multirow{5}{*}{\makecell[l]{Knowledge\\Boundaries}}
& Retriever
& Information retrieval
& Common factual knowledge $\rightarrow$ obscure or synthetic entities
  requiring retrieval from the provided corpus. \\

& Historical Year
& Historical fact lookup
& Well-known event dates $\rightarrow$ obscure or synthetic event years
  unavailable from parametric knowledge. \\

& Game Rule
& Rule/statistic lookup
& Familiar game rules $\rightarrow$ rules or statistics of unfamiliar
  or synthetic games. \\

& Hash
& String hashing
& Familiar hash functions such as MD5 $\rightarrow$ custom or previously
  unseen hashing procedures. \\

& Decoding
& String decoding
& Familiar encodings such as Morse code $\rightarrow$ custom ciphers
  defined externally. \\

\midrule

\multirow{5}{*}{\makecell[l]{Execution\\Reliability}}
& List Manipulation
& Sequential operations
& A few operations on short lists $\rightarrow$ longer or nested list
  transformations. \\

& DateTime
& Date computation
& Simple calendar operations $\rightarrow$ date arithmetic across
  months, years, and leap years. \\

& Code Executor
& Program execution
& Simple expressions $\rightarrow$ loops, recursion, or dynamic-programming
  programs requiring extended execution. \\

& Schedule
& Constraint scheduling
& Simple overlap checks $\rightarrow$ finding feasible intervals among
  multiple interacting events. \\

& Regex Match
& Regex execution
& Simple patterns on short strings $\rightarrow$ complex regular
  expressions over longer inputs. \\

\midrule

\multirow{3}{*}{Multi-hop}
& Chained Calculator
& Dependent computation
& Three-step arithmetic chains with progressively larger and more complex
  intermediate values. \\

& Chained Retriever
& Dependent retrieval
& Multi-step retrieval where one retrieved fact determines the next query. \\

& Chained Code Executor
& Dependent execution
& Three sequential programs in which each output becomes the input to the
  next computation. \\

\bottomrule
\end{tabularx}

\vspace{-2mm}
\end{table*}
This controlled construction is particularly useful for our setting because
minimizing tool calls is not uniformly desirable. On easy instances, avoiding
a redundant call can improve efficiency without affecting correctness, whereas
on hard instances an additional call may be essential for successful task
completion.


\paragraph{Multimodal Safety Alignment.}
For multimodal safety alignment, we follow the setting of
SafeRLHF-V~\citep{ji2025saferlhfv} and use BeaverTails-V for alignment.
BeaverTails-V extends preference-based safety supervision to vision-language
models and contains paired multimodal prompts and responses with separate
annotations for \emph{safety} and \emph{helpfulness}. Each record contains
a textual question, an associated image, an image-severity annotation, two
candidate model responses, the identity of the more helpful response, binary
safety labels for both responses, the safer-response preference, and
fine-grained harmlessness and helpfulness scores. The released data are divided
into 20 fine-grained safety categories organized under nine broader safety
domains, summarized in Table~\ref{tab:beavertails_taxonomy}. The released BeaverTails-V data contain approximately 30K preference examples
for training, while the evaluation portion contains roughly 30 examples per
fine-grained category (approximately 590 in total, with small category-level
variation). Candidate responses were generated using multiple vision-language
models, yielding response pairs that can differ along both safety and
helpfulness dimensions. The dataset therefore naturally exposes the
non-equivalence of the two objectives: a response can become safer through
over-refusal while simultaneously becoming less useful, motivating our choice
of safety as the primary objective and helpfulness as the auxiliary objective. For out-of-distribution safety evaluation, we additionally use
MM-SafetyBench~\citep{liu2024mm}. MM-SafetyBench contains
5,040 text--image pairs across 13 safety-critical scenarios. In contrast to
BeaverTails-V, which provides preference supervision for alignment,
MM-SafetyBench is explicitly designed to evaluate whether multimodal models
remain safe when potentially harmful intent is conveyed through the
interaction between text and images. Its scenarios span illegal activity,
hate speech, malware generation, physical harm, economic harm, fraud,
sexual content, political lobbying, privacy violations, legal advice,
financial advice, health consultation, and government decision-making.
The benchmark further includes several image construction strategies,
including generated images, typography-based images, and combined variants,
to test whether harmful instructions can bypass text-only safety mechanisms.

\begin{table*}[t]
\centering
\caption{
\textbf{Safety taxonomy of BeaverTails-V.}
The dataset contains 20 fine-grained categories grouped into nine broader
safety domains, providing diverse supervision for multimodal safety alignment.
}
\vspace{1mm}
\label{tab:beavertails_taxonomy}

\scriptsize
\setlength{\tabcolsep}{6pt}
\renewcommand{\arraystretch}{1.08}

\begin{tabularx}{\textwidth}{@{}p{5.5cm}Y@{}}
\toprule
\textbf{Safety domain} & \textbf{Fine-grained categories} \\
\midrule

National and public security
& Terrorism or extremism; hacking or digital crime. \\

Health and safety risks
& Harmful health content; environmental damage; dangerous behavior. \\

Physical and psychological harm
& Violence and physical harm; psychological harm and manipulation; animal abuse. \\

Privacy and consent violations
& Privacy invasion and surveillance; identity misuse and impersonation. \\

Discrimination and harassment
& Discriminatory depictions; insulting and harassing behavior. \\

Deceptive and unethical behavior
& Deception in personal relationships; financial and academic fraud. \\

False and sensitive information
& False information; sensitive information in key areas. \\

Disturbing content
& Horror and gore; psychological horror and dark themes. \\

\makecell[l]{Sexually explicit or\\ inappropriate content}
& Pornographic content; sexual crimes. \\

\bottomrule
\end{tabularx}

\vspace{-2mm}
\end{table*}

\begin{table*}[t]
\centering
\caption{
\textbf{Datasets used for multimodal safety alignment and evaluation.}
BeaverTails-V provides safety--helpfulness preference supervision, while
MM-SafetyBench provides an independent multimodal safety evaluation.
}
\vspace{1mm}
\label{tab:safety_dataset_details}
\small
\setlength{\tabcolsep}{5pt}
\renewcommand{\arraystretch}{1.08}

\begin{tabular}{lccp{8.0cm}}
\toprule
\textbf{Dataset}
& \textbf{\# Examples}
& \textbf{Role}
& \textbf{Content / Annotation} \\
\midrule

BeaverTails-V
& $\sim$30K
& Train
& Multimodal preference pairs across 20 safety categories; includes separate
  safety and helpfulness preferences, binary safety labels, image-severity
  levels, and graded response scores. \\

BeaverTails-V
& $\sim$590
& Test
& Held-out examples distributed across the same 20 fine-grained categories,
  generally 28--30 examples per category. \\

MM-SafetyBench
& 5,040
& Test
& Text--image safety benchmark covering 13 safety-critical scenarios and
  multiple image-based attack constructions. \\

\bottomrule
\end{tabular}
\vspace{-2mm}
\end{table*}

\section{Training and Reward Computation Examples}
\label{app:training_examples}

We provide representative examples illustrating the actual training interface
and reward computation for mathematical reasoning and agentic tool use.
The problem inputs and ground-truth answers are taken from the corresponding
training datasets. The assistant responses
below are valid examples constructed according to the actual training format.

\subsection{Mathematical Reasoning Example}
\label{app:math_example}

We first consider a mathematical reasoning example whose ground-truth answer is
$34$.

\paragraph{Input.}

\begin{systembox}
You are a helpful math assistant. Solve the problem step by step.
Put only the final answer in \texttt{\textbackslash boxed\{\}} at the end of
your response. Do not use an ``Answer:'' line; the final answer must be
extractable from \texttt{\textbackslash boxed\{\}}.
\end{systembox}

\begin{userbox}
In triangle $ABC$, $\sin \angle A = \frac{4}{5}$ and
$\angle A < 90^\circ$. Let $D$ be a point outside triangle $ABC$ such that
$\angle BAD = \angle DAC$ and $\angle BDC = 90^\circ$.
Suppose that $AD = 1$ and that
$\frac{BD}{CD} = \frac{3}{2}$.
If $AB + AC$ can be expressed in the form
$\frac{a\sqrt{b}}{c}$ where $a,b,c$ are pairwise relatively prime integers,
find $a+b+c$.
\end{userbox}

The reference answer is
\[
y^*=34.
\]

\paragraph{Example response.}

\begin{assistantbox}
Since $\sin A=\frac45$ and $A<90^\circ$, we have
\[
\cos A=\frac35.
\]

Because $AD$ bisects $\angle A$, let
\[
\theta=\frac A2.
\]
Then
\[
\cos\theta
=
\sqrt{\frac{1+\cos A}{2}}
=
\sqrt{\frac{1+\frac35}{2}}
=
\frac{2}{\sqrt5},
\qquad
\sin\theta=\frac{1}{\sqrt5}.
\]

Place
\[
A=(0,0),\qquad D=(1,0),
\]
with $AD$ on the $x$-axis, and write
\[
AB=m,\qquad AC=n.
\]
Then
\[
B=
\left(
\frac{2m}{\sqrt5},
\frac{m}{\sqrt5}
\right),
\qquad
C=
\left(
\frac{2n}{\sqrt5},
-\frac{n}{\sqrt5}
\right).
\]

Since $\angle BDC=90^\circ$,
\[
(B-D)\cdot(C-D)=0,
\]
and hence
\[
\left(\frac{2m}{\sqrt5}-1\right)
\left(\frac{2n}{\sqrt5}-1\right)
-\frac{mn}{5}
=
0.
\]
This simplifies to
\[
\frac35mn-\frac{2}{\sqrt5}(m+n)+1=0.
\]

Furthermore,
\[
BD^2=m^2-\frac{4m}{\sqrt5}+1,
\qquad
CD^2=n^2-\frac{4n}{\sqrt5}+1.
\]
Since
\[
\frac{BD}{CD}=\frac32,
\]
we obtain
\[
4\left(
m^2-\frac{4m}{\sqrt5}+1
\right)
=
9\left(
n^2-\frac{4n}{\sqrt5}+1
\right).
\]

Solving the two equations gives
\[
(m,n)
=
\left(
\frac{\sqrt5}{6},
\frac{4\sqrt5}{9}
\right)
\]
or
\[
(m,n)
=
\left(
\frac{7\sqrt5}{6},
\frac{8\sqrt5}{9}
\right).
\]

The second configuration places $D$ inside triangle $ABC$, whereas the problem
requires $D$ to lie outside. Therefore,
\[
AB+AC
=
\frac{\sqrt5}{6}
+
\frac{4\sqrt5}{9}
=
\frac{11\sqrt5}{18}.
\]

Thus,
\[
a=11,\qquad b=5,\qquad c=18,
\]
and
\[
a+b+c=34.
\]

\[
\boxed{34}
\]
\end{assistantbox}

\paragraph{Reward computation.}

The last valid boxed answer is $\boxed{34}$, which matches the ground truth.
The response is also shorter than the $8{,}000$-token length threshold.
Therefore,
\[
R_{\mathrm{acc}}=1,
\qquad
R_{\mathrm{len}}=1.
\]



\subsection{Agentic Tool-Use Examples}
\label{app:tool_examples}

We provide two multi-step tool-use examples that illustrate different forms of
agentic dependency. The first requires iterative code execution, where later
programs must be instantiated from earlier execution results. The second
requires multi-hop retrieval, where later search queries depend on entities
obtained in previous retrieval steps.

\subsubsection{Three-Step Code Execution Example}
\label{app:tool_code_example}

We first consider a hard multi-step code-execution task. The expected
intermediate results are
\[
x=987,\qquad y=549,\qquad z=93,
\]
and the final ground-truth answer is
\[
y^*=93.
\]
The task-specific tool budget is
\[
B=3.
\]

\paragraph{Input.}

\begin{systembox}
You can use the provided tools when helpful, but tool use is optional.
Call at most one tool at a time and use arguments that match its schema.
When you have the answer, end with exactly one final answer in the form
\texttt{\textbackslash boxed\{...\}}.
\end{systembox}

\begin{userbox}
Three-step code execution: Step 1: Run the following code. Call the result x.

\begin{lstlisting}[style=promptcode]
memo = {}
def fib(n):
    if n in memo: return memo[n]
    if n < 2: return n
    memo[n] = fib(n-1) + fib(n-2)
    return memo[n]
print(fib(16))
\end{lstlisting}

Step 2: Using x, run: compute sum of all proper divisors of x
(excluding x itself). Call the result y.

Step 3: Using y, run: compute min coins for amount=y with
coins=[1, 4, 6]. Call the result z.

Return only z.
\end{userbox}

The instance exposes a single code-execution tool:

\begin{toolbox}
\begin{lstlisting}[style=promptcode]
{
  "type": "function",
  "function": {
    "name": "run_code",
    "description":
      "Execute a Python code snippet and return its stdout output.",
    "parameters": {
      "type": "object",
      "properties": {
        "code": {
          "type": "string",
          "description": "Python code to execute."
        }
      },
      "required": ["code"]
    }
  }
}
\end{lstlisting}
\end{toolbox}

\paragraph{Step 1: computing the Fibonacci value.}

The first program is directly provided by the user, so the agent executes it:

\begin{assistantbox}
\begin{lstlisting}[style=promptcode]
<tool_call>
{"name":"run_code",
 "arguments":{
   "code":"memo = {}\ndef fib(n):\n    if n in memo: return memo[n]\n    if n < 2: return n\n    memo[n] = fib(n-1) + fib(n-2)\n    return memo[n]\nprint(fib(16))"
 }}
</tool_call>
\end{lstlisting}
\end{assistantbox}

The execution environment returns:

\begin{toolbox}
\begin{lstlisting}[style=promptcode]
{
  "success": true,
  "stdout": "987",
  "stderr": "",
  "exit_code": 0,
  "result": "987"
}
\end{lstlisting}
\end{toolbox}

The agent therefore obtains
\[
x=987.
\]

\paragraph{Step 2: summing the proper divisors.}

The second program must be constructed using the execution result from the
previous step:

\begin{assistantbox}
\begin{lstlisting}[style=promptcode]
<tool_call>
{"name":"run_code",
 "arguments":{
   "code":"x = 987\ns = 0\nfor d in range(1, x):\n    if x % d == 0:\n        s += d\nprint(s)"
 }}
</tool_call>
\end{lstlisting}
\end{assistantbox}

The tool returns:

\begin{toolbox}
\begin{lstlisting}[style=promptcode]
{
  "success": true,
  "stdout": "549",
  "stderr": "",
  "exit_code": 0,
  "result": "549"
}
\end{lstlisting}
\end{toolbox}

Hence,
\[
y=549.
\]

For reference,
\[
987=3\times7\times47,
\]
whose proper divisors are
\[
1,3,7,21,47,141,329,
\]
and therefore
\[
1+3+7+21+47+141+329=549.
\]

\paragraph{Step 3: computing the minimum number of coins.}

The final program is dynamically instantiated using $y=549$:

\begin{assistantbox}
\begin{lstlisting}[style=promptcode]
<tool_call>
{"name":"run_code",
 "arguments":{
   "code":"y = 549\ncoins = [1, 4, 6]\ndp = [float(\"inf\")] * (y + 1)\ndp[0] = 0\nfor amount in range(1, y + 1):\n    for coin in coins:\n        if coin <= amount:\n            dp[amount] = min(dp[amount], dp[amount - coin] + 1)\nprint(dp[y])"
 }}
</tool_call>
\end{lstlisting}
\end{assistantbox}

The execution environment returns:

\begin{toolbox}
\begin{lstlisting}[style=promptcode]
{
  "success": true,
  "stdout": "93",
  "stderr": "",
  "exit_code": 0,
  "result": "93"
}
\end{lstlisting}
\end{toolbox}

The agent therefore obtains
\[
z=93.
\]

Since the user requests only the final value, the assistant terminates with:

\begin{assistantbox}
\[
\boxed{93}
\]
\end{assistantbox}

\paragraph{Reward computation.}

The final boxed answer matches the ground truth, so
\[
R_{\mathrm{pri}}=1.
\]

The trajectory executes three tool calls:
\[
N_{\mathrm{call}}=3.
\]
For this multi-step instance,
\[
B=3,
\qquad
\delta=1.
\]
Therefore,
\[
R_{\mathrm{aux}}
=
\max\left(
0,\,
1-\frac{N_{\mathrm{call}}}{B+\delta}
\right)
=
1-\frac{3}{3+1}
=
0.25.
\]

The raw reward vector is
\[
(R_{\mathrm{pri}},R_{\mathrm{aux}})
=
(1,0.25).
\]

This example illustrates a genuinely agentic dependency structure. Although
the first program is given directly in the prompt, the second program cannot
be instantiated until the first execution produces $x=987$, and the third
cannot be instantiated until the second produces $y=549$. The agent must
therefore maintain intermediate execution state and dynamically construct later
tool arguments from earlier tool observations.

\subsubsection{Three-Hop Retrieval Example}
\label{app:tool_multihop_example}

We next consider a three-hop retrieval task from
\texttt{RetrieverEnv}. The task requires each subsequent query to be
constructed from the result of the previous retrieval step.

The dependency chain is
\[
\textit{Moby-Dick}
\rightarrow
\text{Herman Melville}
\rightarrow
\text{United States}
\rightarrow
\text{Washington, D.C.},
\]
with ground-truth answer
\[
y^*=\text{Washington, D.C.}
\]
and tool budget
\[
B=3.
\]

\paragraph{Input.}

\begin{systembox}
You can use the provided tools when helpful, but tool use is optional.
Call at most one tool at a time and use arguments that match its schema.
When you have the answer, end with exactly one final answer in the form
\texttt{\textbackslash boxed\{...\}}.
\end{systembox}

\begin{userbox}
Three-step retrieval: Step 1: Who wrote Moby-Dick? Call the answer x.
Step 2: In which country was x born? Call the answer y.
Step 3: What is the capital of y? Call the answer z.
Return only z.
\end{userbox}

The instance exposes two local retrieval tools,
\texttt{search\_corpus} and \texttt{read\_doc}.

\begin{toolbox}
\begin{lstlisting}[style=promptcode]
[
  {
    "type": "function",
    "function": {
      "name": "search_corpus",
      "description":
        "Search local corpus and return metadata-only hits
         (id/title/score). Use read_doc to access full text.",
      "parameters": {
        "type": "object",
        "properties": {
          "query": {
            "type": "string",
            "description": "Search query text."
          },
          "top_k": {
            "type": "integer",
            "description":
              "Maximum number of hits to return."
          }
        },
        "required": ["query"]
      }
    }
  },
  {
    "type": "function",
    "function": {
      "name": "read_doc",
      "description": "Read a full document by id.",
      "parameters": {
        "type": "object",
        "properties": {
          "doc_id": {
            "type": "string",
            "description":
              "Document id from search results."
          }
        },
        "required": ["doc_id"]
      }
    }
  }
]
\end{lstlisting}
\end{toolbox}

\paragraph{Step 1: identifying the author.}

The agent first searches for \textit{Moby-Dick}:

\begin{assistantbox}
\begin{lstlisting}[style=promptcode]
<tool_call>
{"name":"search_corpus",
 "arguments":{"query":"Moby-Dick","top_k":1}}
</tool_call>
\end{lstlisting}
\end{assistantbox}

\begin{toolbox}
\begin{lstlisting}[style=promptcode]
{
  "success": true,
  "data": {
    "query": "Moby-Dick",
    "total_hits": 10,
    "returned": 1,
    "hits": [
      {
        "id": "doc1",
        "title": "Moby-Dick",
        "snippet":
          "Moby-Dick is an 1851 novel by Herman Melville
           about Captain Ahab's quest to hunt the white whale.",
        "score": 1
      }
    ]
  }
}
\end{lstlisting}
\end{toolbox}

From this result, the agent obtains
\[
x=\text{Herman Melville}.
\]

\paragraph{Step 2: identifying the birth country.}

The second query is generated from the entity obtained in Step 1:

\begin{assistantbox}
\begin{lstlisting}[style=promptcode]
<tool_call>
{"name":"search_corpus",
 "arguments":{
   "query":"Herman Melville biography",
   "top_k":1
 }}
</tool_call>
\end{lstlisting}
\end{assistantbox}

\begin{toolbox}
\begin{lstlisting}[style=promptcode]
{
  "success": true,
  "data": {
    "query": "Herman Melville biography",
    "total_hits": 10,
    "returned": 1,
    "hits": [
      {
        "id": "doc10",
        "title": "Herman Melville biography",
        "snippet":
          "Herman Melville was born in New York City,
           United States, in 1819.",
        "score": 3
      }
    ]
  }
}
\end{lstlisting}
\end{toolbox}

The agent therefore obtains
\[
y=\text{United States}.
\]

\paragraph{Step 3: identifying the capital.}

The final query is constructed from the country obtained in Step 2:

\begin{assistantbox}
\begin{lstlisting}[style=promptcode]
<tool_call>
{"name":"search_corpus",
 "arguments":{
   "query":"United States capital",
   "top_k":1
 }}
</tool_call>
\end{lstlisting}
\end{assistantbox}

\begin{toolbox}
\begin{lstlisting}[style=promptcode]
{
  "success": true,
  "data": {
    "query": "United States capital",
    "total_hits": 10,
    "returned": 1,
    "hits": [
      {
        "id": "doc7",
        "title": "United States capital",
        "snippet":
          "The capital of the United States is
           Washington, D.C., established in 1790
           as the seat of the federal government.",
        "score": 3
      }
    ]
  }
}
\end{lstlisting}
\end{toolbox}

Thus,
\[
z=\text{Washington, D.C.}
\]

Since the user requests only the final entity, the assistant terminates with:

\begin{assistantbox}
\[
\boxed{\text{Washington, D.C.}}
\]
\end{assistantbox}

\paragraph{Reward computation.}

The final boxed answer matches the reference answer, so
\[
R_{\mathrm{pri}}=1.
\]

The trajectory contains three tool calls:
\[
N_{\mathrm{call}}=3.
\]
For this multi-step instance,
\[
B=3,
\qquad
\delta=1.
\]
Therefore,
\[
R_{\mathrm{aux}}
=
\max\left(
0,\,
1-\frac{N_{\mathrm{call}}}{B+\delta}
\right)
=
1-\frac{3}{3+1}
=
0.25.
\]

The raw reward vector is
\[
(R_{\mathrm{pri}},R_{\mathrm{aux}})
=
(1,0.25).
\]

\paragraph{Multi-hop dependency.}

This task requires multi-turn state propagation. The second query
cannot be instantiated until the first retrieval identifies Herman Melville,
and the third query cannot be instantiated until the second retrieval
identifies the United States. The agent must therefore maintain the
intermediate state
$x\rightarrow y\rightarrow z,$
dynamically construct later tool arguments from earlier observations, and
respect the one-tool-call-per-turn constraint. The final response, meanwhile,
must expose only the requested terminal entity rather than the intermediate
retrieval process.

Together, the two tool-use examples illustrate complementary forms of
long-horizon agentic behavior. The code-execution example requires dynamically
constructing executable programs from intermediate numerical states, whereas
the retrieval example requires dynamically constructing semantic queries from
intermediate entities. In both cases, task correctness serves as the primary
objective, while tool-use efficiency provides auxiliary supervision over how
the correct result is obtained.

\end{document}